\documentclass[journal]{IEEEtran}
\usepackage{mathtools}
\usepackage{amsmath}
\usepackage{amsthm}
\usepackage{xcolor}
\allowdisplaybreaks
\usepackage{amsfonts}
\usepackage{graphics}
\usepackage{subcaption}
\theoremstyle{definition}
\newtheorem{definition}{Definition}
\newtheorem{theorem}{Theorem}

\newtheorem{remark}{Remark}
\newtheorem{lemma}{Lemma}
\usepackage{algorithm}
\usepackage{algpseudocode}
\usepackage{graphicx}
\graphicspath{ {./images/} }
\usepackage{eqparbox}
\newcommand{\quotes}[1]{``#1''}
\newdimen{\algindent}
\algnewcommand\LeftComment[2]{
\Statex \hspace{#1\algindent}\(\triangleright\) \eqparbox{}{\textit{#2}} \hfill 
}

\usepackage{biblatex}
\bibliography{reference}

\begin{document}
\title{AFed: Algorithmic Fair Federated Learning}

\author{Huiqiang~Chen,~
        Tianqing~Zhu,~~\IEEEmembership{Member,~IEEE}
        Wanlei~Zhou,~~\IEEEmembership{Fellow,~IEEE}
        and~Wei~Zhao,~\IEEEmembership{Fellow,~IEEE}
\thanks{This research is partially supported by NSFC-FDCT under its Joint Scientific Research Project Fund (Grant No. 0051/2022/AFJ).}
\thanks{Huiqiang Chen, Tianqing Zhu, and Wanlei Zhou are with the Faculty of Data Science, City University of Macau, Macao, China (e-mail:cs.hqchen@gmail.com; tqzhu@cityu.edu.mo; wlzhou@cityu.edu.mo).}
\thanks{Wei Zhao is with Shenzhen University of Advanced Technology, Shenzhen, China. (e-mail: weizhao86@outlook.com).}
}

\markboth{Journal of \LaTeX\ Class Files,~Vol.~14, No.~8, August~2015}%
{Shell \MakeLowercase{\textit{et al.}}: Bare Demo of IEEEtran.cls for IEEE Journals}

\maketitle

\begin{abstract}
Federated Learning (FL) has gained significant attention as it facilitates collaborative machine learning among multiple clients without centralizing their data on a server. FL ensures the privacy of participating clients by locally storing their data, which creates new challenges in fairness. Traditional debiasing methods assume centralized access to sensitive information, rendering them impractical for the FL setting. Additionally, FL is more susceptible to fairness issues than centralized machine learning due to the diverse client data sources that may be associated with group information. Therefore, training a fair model in FL without access to client local data is important and challenging. This paper presents AFed, a straightforward yet effective framework for promoting group fairness in FL. The core idea is to circumvent restricted data access by learning the global data distribution. This paper proposes two approaches: AFed-G, which uses a conditional generator trained on the server side, and AFed-GAN, which improves upon AFed-G by training a conditional GAN on the client side. We augment the client data with the generated samples to help remove bias. Our theoretical analysis justifies the proposed methods, and empirical results on multiple real-world datasets demonstrate a substantial improvement in AFed over several baselines.
\end{abstract}


\section{Introduction}
\IEEEPARstart{W}{}ith the increasing use of machine learning algorithms in critical decision-making areas such as credit evaluation, loan applications, and healthcare, there are concerns about potential bias and discrimination in trained models. For instance, COMPAS, a software used by courts to aid judges in pretrial detention and release decisions, has been found to have a substantial bias against African-Americans when comparing false positive rates of African-American offenders to Caucasian offenders.

As a result, fairness research has gained prominence, with group fairness \cite{hardt2016equality} and individual fairness \cite{dwork2012fairness} being two broad classifications for approaches to fairness. We refer to these two concepts as algorithmic fairness because the definition of fairness is expanded in the context of FL, such as \textit{accuracy parity} \cite{li2019fair}, \textit{good-intent fairness} \cite{mohri2019agnostic,du2021fairness}. This paper presents an algorithmic fair FL framework emphasizing group fairness since clients tend to belong to specific demographic groups in FL \cite{kairouz2019advances}.

Various algorithms have been proposed to train an algorithmic fair model in centralized machine learning \cite{zhang2021balancing,zhang2020fairness,zhang2019fairness,zafar2017fairness}, including \textit {pre-processing}, \textit{in-processing} and \textit{post-processing} \cite{d2017conscientious}. Although the details differ, all of these algorithms assume centralized access to data. In FL, however, training data is not centrally accessible by either the server or the clients. Therefore, one cannot simply apply centralized fair learning algorithms in FL tasks. This begs the question: \textbf{How to train an algorithmic fair FL model without centralized access to data?}

This problem is challenging due to restricted access and data heterogeneity. FL models are trained on all clients' local data, with only gradients/model parameters transmitted to the server during each round. Debiasing on the server side is not feasible since the server is prohibited from directly accessing clients' local data. On the other hand, a single client can't accurately picture the global distribution due to limited data \cite{mcmahan2017communication}. Worse yet, data distributions vary between clients \cite{kairouz2019advances}. There is no guarantee that the model debiased with local data will generalize to the global distribution. In addition, it has been long observed that fairness conflicts with accuracy \cite{zhao2022inherent,berk2017convex}. A trivial model that makes constant predictions is perfectly fair but useless.

To address these challenges, we proposed a novel framework named AFed to achieve algorithmic fair FL. AFed aims to learn a fair model satisfying group fairness (e.g., demographic parity \cite{hardt2016equality}) in FL. The core idea is to learn global data distribution and then disseminate that knowledge to all clients to help with local debiasing. To this end, two algorithms are proposed. Our first algorithm, AFed-G, trains a conditional generator to learn clients' data distribution on the server side. AFed-G enjoys less computation and communication overhead. The second algorithm, AFed-GAN, trains a conditional GAN on the client side, which is more effective in learning data distribution. An auxiliary classification head is designed to help extract informative features that benefit the generator training. The knowledge about global distribution is then shared among all clients to help debiasing the model. We mix the true training data with generated data to debiase the model. 

In summary, the contributions of this paper are listed below.
\begin{itemize}
    \item We tackle the challenge of training a fair model in FL without accessing clients' data. We bypass the restricted centralized data access by introducing a conditional generator/GAN.
    \item We design an auxiliary classification head to help extract more informative features, which benefits the conditional generator/GAN training.
    \item We provide theoretical analysis to justify the robustness of the proposed method. Our empirical results demonstrate substantial improvement of the proposed methods over baselines on several real-world datasets.
\end{itemize}

\section{Related works}
\subsection{Fair Machine learning}
At a high level, algorithmic fairness notions adopted in centralized machine learning can be categorized as two different families: the \textit{statistical} notion and \textit{individual} notion \cite{chouldechova2018frontiers}. The statistical notions require specific statistical measures to be comparable across all of these groups \cite{6137441,calders2010three}. The individual notions aim at fairness between specific pairs of individuals: \textit{Give similar predictions to similar individuals} \cite{dwork2012fairness}. Individual notions offer a meaningful fairness guarantee but need to make significant assumptions, some of which present non-trivial challenges in fairness. We adopt statistical notions, more specifically, demographic parity \cite{hardt2016equality} as fairness metric in this paper.

Debiasing methods can be classified as \textit {pre-processing}, \textit{in-processing}, and \textit{post-processing}, respectively \cite{d2017conscientious}. Pre-processing improves fairness by eliminating bias present in the training dataset. The model is subsequently trained and applied on the modified data \cite{feldman2015certifying,kamiran2012data}. For instance, Kamiran et al. \cite{kamiran2012data} propose reweighting samples in the training dataset to correct for biased treatment. Xu et al. \cite{xu2018fairgan} proposed FairGAN, which generates fair data from the original training data and uses the generated data to train the model. In-processing achieves fairness by adding constraints during training. Berk et al. \cite{berk2017convex} integrated a family of fairness regularizers into the objective function for regression issues. Zhang et al. \cite{zhang2018mitigating} borrowed the ideal of adversarial training to restrict a model's bias. The general methodology of post-processing methods is to adjust the decision thresholds of different groups following specific fairness standards \cite{calders2010three,hardt2016equality,bolukbasi2016man}. However, existing methods assume centralized access to the data, which is invalid in FL. Each client can only access its own data, and the server does not know the local data distribution. As such, it remains challenging to train a fair model in FL.

\subsection{Fairness in Federated Learning}
Fairness can be defined from different aspects in FL. A body of work focuses on accuracy disparity, i.e., trying to minimize the performance gaps between different clients \cite{li2019fair}. A model with a more uniform performance distribution is considered fairer. The goal, in this case, is to address any statistical heterogeneity during the training phase by using techniques like data augmentation \cite{zhao2018federated}, client selection \cite{yang2020federated} and reweighting \cite{ huang2020fairness}, multi-task learning \cite{smith2017federated}, etc. Another area of focus is good-intent fairness \cite{mohri2019agnostic}, ensuring that the training procedure does not overfit a model to any client at the expense of others. In this instance, the objective is to train a robust model against an unknown testing distribution. 

However, these two lines of work don't address the algorithmic fairness issues in FL. To date, there are limited works about algorithmic fairness in FL. Du et al. \cite{du2021fairness} formulated a distributed optimization problem with fairness constraints. Gálvez et al. \cite{Galvez2021EnforcingFI} cast fairness as a constraint to an optimization problem. To protect privacy, the constraint problem is optimized based on statistics provided by clients. Chu et al. \cite{Chu2021FedFairTF} and  Cui et al. \cite{cui2021addressing} took a similar path, imposing fairness constraints on the learning problem. To avoid intruding on the data privacy of any client, the fairness violations are locally calculated by each client. Zhang et al. \cite{zhang2020fairfl} focused on client selection, using a team Markov game to select clients to ensure fairness.  Ezzeldin et al. \cite{ezzeldin2021fairfed} proposed a fairness-aware aggregation method that amplifies local debiasing by weighting clients whose fairness metric aligns well with the global one. Wang et al. \cite{wang2023mitigating} explore the relationship between local group fairness and propose a regularized objective for achieving global fairness. Zeng et al. \cite{zeng2021improving} extend FairBatch \cite{roh2020fairbatch}, a fair training method in centralized machine learning, to FL. In their design, each client shares extra information about the unfairness of its local classifier with the server. Chu et al. \cite{chu2021fedfairtrainingfairmodels} propose estimating global fairness violations by aggregating the fairness assessments made locally by each client. They then constrain the model training process based on the global fairness violation.

\subsection{Discussion of Related Works}
Currently, most FL research focuses on accuracy disparity, which is more closely related to data heterogeneity rather than the potentially unfair decisions the model may make. Comparatively, less research has been conducted regarding algorithmic fairness in FL. The difficulty of ensuring algorithmic fair FL stems from the privacy restriction of FL, i.e., the data never leaves the device. In the presence of client-specific data heterogeneity, however, local debiasing may fail to provide acceptable performance for the entire population. Existing works have attempted to circumvent this restriction by providing server statistics or calculating fairness constraints locally. However, they fail to adequately address the non-i.i.d. challenge in FL. For example, in the presence of non-i.i.d. data, the estimated global fairness violation in \cite{chu2021fedfairtrainingfairmodels} could be inaccurate. Our approach distinguishes itself from existing methods in two key ways: 1) Alternative methodology: We explore a different direction by leveraging local data augmentation to enhance debiasing efforts; 2) Enhanced practicality: We address the non-i.i.d. challenges inherent in FL when training fair models. Our local data augmentation technique effectively mitigates these non-i.i.d. issues.

\section{Background}
\subsection{Notations}
This research focuses on addressing fairness issues in the context of binary classification tasks to illustrate our method. Note that our methods can be extended to the general multi-class classification problem since we didn't explicitly require the target label to be binary. Denote $\mathcal X \in \mathbb R^m$ as the input space. $\mathcal Z \in \mathbb R^p$ is the latent feature space with $p<m$. $\mathcal{Y} = \left\{0,1\right\}$ is the output space, and $\mathcal{A} = \left\{0,1\right\}$ is the set of all possible sensitive attribute values. The classification model $f$ parameterized by $\boldsymbol{\theta}^f \coloneqq \boldsymbol{\theta}^E \circ \boldsymbol{\theta}^{y}$ consists of two parts: a feature extractor $E: \mathcal X \mapsto \mathcal Z$ and a classification head $h^y: \mathcal Z \mapsto \Delta^\mathcal Y$, where $\Delta^\mathcal Y$ stands for the probability simplex over $\mathcal Y$. In addition, we train another classification head $h^a: \mathcal Z \mapsto \Delta^\mathcal A$ parameterized by $\boldsymbol{\theta}^{a}$ to classify sensitive attribute $a$. In AFed-G, a conditional generator $G$ parameterized by $\boldsymbol{w}^G$ is trained at the server side to synthesize fake representations condition on $a$: $\mathcal A \mapsto \mathcal Z$. In AFed-GAN, we replace $G$ with conditional GAN parameterized by $\boldsymbol{w} \coloneqq \boldsymbol{w}^G \circ \boldsymbol{w}^D$. The discriminator $D: \boldsymbol{z} \times a \mapsto \left[0, 1\right]$ is trained to distinguish fake samples from real ones.

\subsection{Federated Learning}
A FL system comprises $N$ local clients and a central server. Each client possesses a specific dataset $\mathcal{D}^k$, $k\in \left\{1,2,...,N\right\}$ and the $i$-th data sample is represented as $\xi_i \coloneqq \left\{\boldsymbol{x}_i, y_i, a_i\right\}$, where $\boldsymbol{x}_i $ is the features, $y_i$ is the ground-true label, and $a_i$ is the sensitive attribute. The amount of data of the $k$-th client is denoted as $n_k$. The goal of FL is to minimize the empirical risk over the samples of all clients, i.e.,
\begin{equation}
    \boldsymbol{\theta}^f = \mathop{\arg\min} \limits_{\boldsymbol{\theta}^f} \frac{1}{\sum_{k=1}^N n_k} \sum^N_{k=1}\sum^{n_k}_{i=1}\ell_k(f(\boldsymbol{x}_i;\boldsymbol{\theta}^f),y_i),
\end{equation}
where $\ell_k$ is the loss objective of the $k$-th client. Unlike central machine learning, which assumes centralized access to all train data $\mathcal{D} = \cup_{k \in N} \mathcal{D}_k$, the client's training data is never exposed to the server. Only gradients/model parameters are transmitted to the server for aggregation. 

At the $t$-th training round,  the server randomly selects a set of clients $\mathcal{S}$ to train the current global model $\boldsymbol{w}^{t}$. Selected clients first receive $\boldsymbol{w}^{t}$ from the server, then train it on their local dataset $\mathcal{D}^k, k\in \mathcal{S}$,
\begin{equation} \label{eq: sgd}
    \boldsymbol{\theta}_k^{t+1} = \boldsymbol{\theta}^{t}- \frac{\eta}{b_k^t}\sum \limits_{i\in \mathcal{B}^k_t} \nabla \ell_k(f(\boldsymbol{x}_i;\boldsymbol{\theta}_k),y_i),
\end{equation}
where $\eta$ denotes the learning rate and $\mathcal{B}_t^k$ is a mini-batch of training data sampled from $\mathcal{D}_k$ in the $t$-th iteration. The local training process could run for multiple rounds. After that, the trained model $\boldsymbol{\theta}_k^{t}$ or equivalently gradients $\Delta \boldsymbol{\theta} \coloneqq \boldsymbol{\theta}_k^{t+1}-\boldsymbol{\theta}^{t}$ is send back to the server for aggregation,
\begin{equation} \label{eq: avg}
     \boldsymbol{\theta}^{t+1} = \sum \limits_ {k \in \mathcal{S}}\omega_k \boldsymbol{\theta}_k^{t+1},
\end{equation}
where $\omega_k =n_k/{\sum_{k\in\mathcal{S}}n_k}$ is the aggregation weight. $\boldsymbol{\theta}^{t+1}$ is used as the initial model for the next training round. This interaction between clients and the server is repeated until a certain criterion (e.g., accuracy) is satisfied. 

\subsection{Fairness}
Fairness in FL can be defined in different ways. This paper focuses on algorithmic fairness, especially group fairness, which is measured by the disparities in the algorithm decisions made for different groups determined by sensitive attributes, such as gender, race, etc. Here, we have opted for demographic parity as the fairness metric in this paper. 
\begin{definition}[Demographic parity \cite{hardt2016equality}] Demographic parity (DP) requires that a model's prediction be independent of any sensitive attributes. The positive prediction rate for a demographic group $a \in \mathcal{A}$ is defined as follows,
\begin{gather}
    \gamma_0(\hat Y) =\textrm{Pr}\left(\hat Y=1| A=0\right) = \mathbb{E}_{x \sim P_0}\left[f(x)\right]\\
    \gamma_1(\hat Y) = \textrm{Pr}\left(\hat Y=1| A=1\right)=\mathbb{E}_{x \sim P_1}\left[f(x)\right]
\end{gather}
\end{definition}
The task of fair supervised learning is to train a classification model $f$ that can accurately label samples while minimizing discrimination. The discrimination level in terms of demographic parity can be measured by the absolute difference in the model's outcomes between groups
\begin{equation}\label{eq: dpreg}
    \Delta \textrm{DP} = |\gamma_0(\hat Y) - \gamma_1(\hat Y)|
\end{equation}
The smaller $\Delta \textrm{DP}(f)$ is, the fairer $f$ will be. $\Delta \textrm{DP}(f) =0 $ iif $f(x) \perp A$, i.e., demographic parity is achieved.

\subsection{GAN and Conditional GAN}
GAN was first proposed by \cite{goodfellow2014generative} as a min-max game between the generator $G$ and discriminator $D$. The aim is to approximate a target data distribution so that the generated data can't be separated from the true data by $D$. The objective function of the generator is formed as
\begin{equation}
    \mathop{\min \limits_G \max \limits_D} \mathbb{E}_{\boldsymbol{x}\sim p(x)} \left[\log D(x) + \log(1-D(G(x)))\right],
\end{equation}
where $p(x)$ is the real data distribution.

There exist various kinds of generators to generate synthetic data for different purposes. Among these, we chose conditional GAN \cite{mirza2014conditional} to generate synthetic data conditional on the given input. The conditional GAN differs from the unconditional GAN by providing the generator and discriminator condition information. The learning objective of conditional GAN is formulated as follows:
\begin{equation}
    \mathop{\min \limits_G \max \limits_D} \mathbb{E}_{x\sim p(x)} \left[\log D(x,a) + \log(1-D(G(x,a)))\right].
\end{equation}
The additional information $a$ is fed to the generator and the discriminator during training. After convergence, the generator can generate fake data conditional on $a$.

\begin{figure*}[htp]
    \centering
    \begin{subfigure}[b]{0.24\textwidth}
        \includegraphics[width=\textwidth]{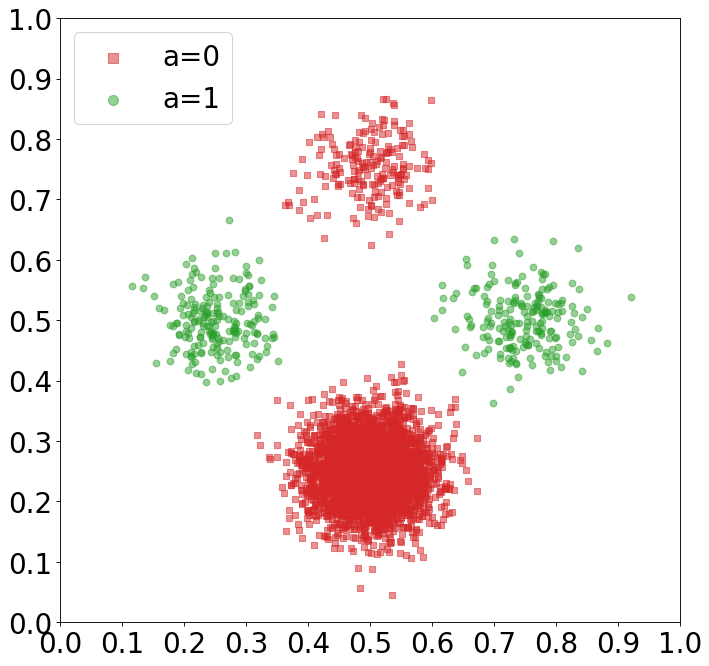}
        \caption{Client 1's data}
        \label{fig: client1 data}
    \end{subfigure}
    \hfill
    \begin{subfigure}[b]{0.24\textwidth}
        \includegraphics[width=\textwidth]{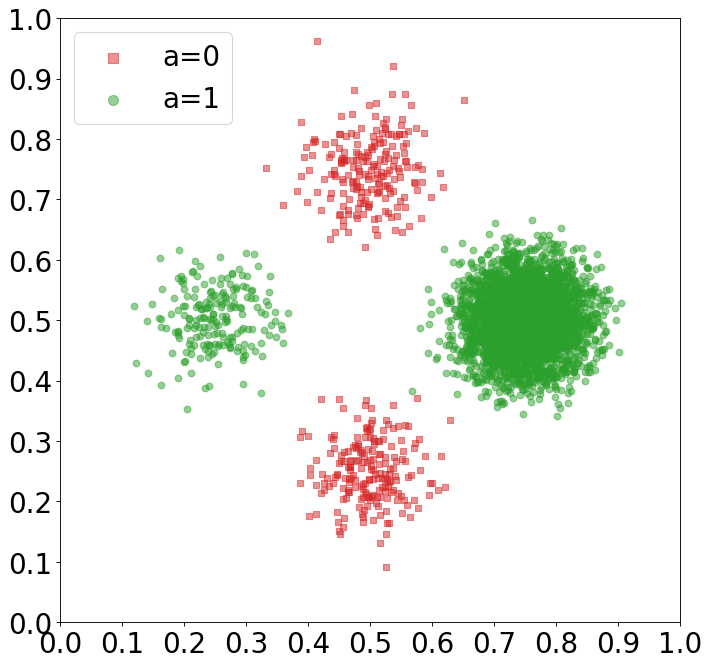}
        \caption{Client 2's data}
        \label{fig: client2 data}
    \end{subfigure}
    \hfill
    \begin{subfigure}[b]{0.24\textwidth}
        \includegraphics[width=\textwidth]{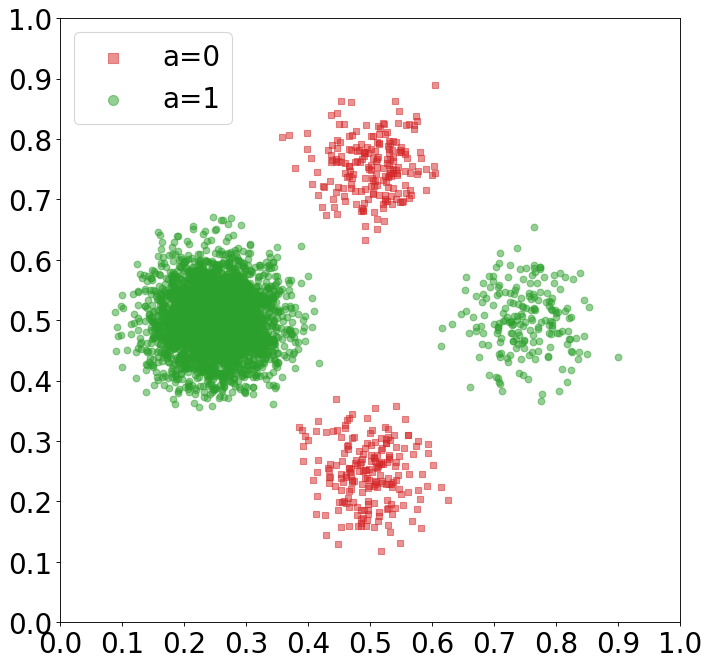}
        \caption{Client 3's data}
        \label{fig: client3 data}
    \end{subfigure}
    \hfill
    \begin{subfigure}[b]{0.24\textwidth}
        \includegraphics[width=\textwidth]{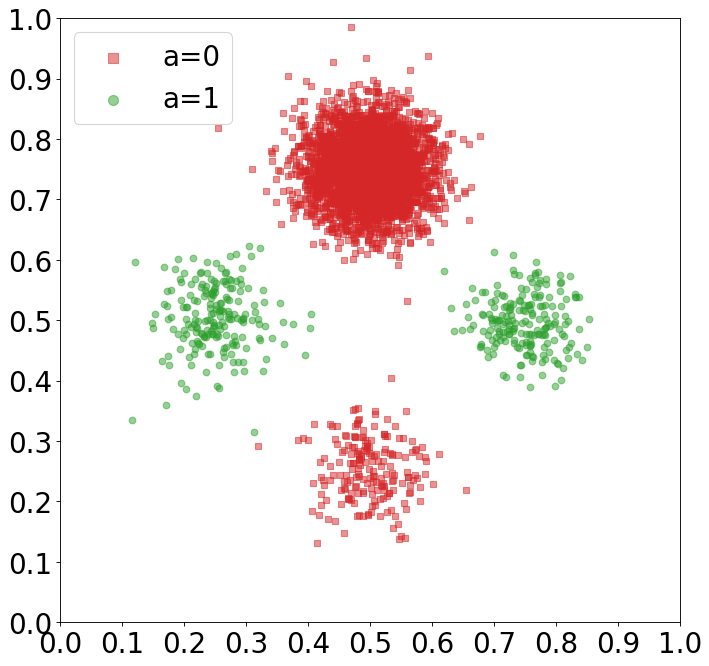}
        \caption{Client 4's data}
        \label{fig: client4 data}
    \end{subfigure}
    \caption{The data distribution of each client in the toy example. Data is sampled from a mixture of four Gaussian distributions. For each client, 85\% of the data is sampled from a single Gaussian distribution, and the rest 15\% is evenly sampled from the other three Gaussian distributions.}
\end{figure*}
\section{Algorithmic Fair Federated Learning}
\subsection{Motivation Example}
The non-i.i.d. nature of local data distributions in a typical FL system presents a significant challenge for local debiasing. Certain groups may dominate a client's data, while others may have limited representation. To illustrate this challenge, we create a toy FL system with four clients, each with a unique local dataset having different distributions. We generated the data from a mixture of four Gaussian distributions and grouped them based on the attribute $a$ into two categories. In this example, 85\% of a client's data is sampled from one Gaussian distribution, while the remaining 15\% is evenly distributed across the other three Gaussian distributions.

In this situation, local debiasing is futile because none of the clients completely understand the data distribution across all clients. For example, client 1 has very few samples from group 1 (i.e., $a=1$), as shown in Fig. \ref{fig: client1 data}. Even less is the proportion when all clients' data is considered. Therefore, it is difficult for client 1 to debias the classification model, given that client 1 has almost no knowledge of group 1. The same issue arises with the other clients. Our theoretical analysis in Section V confirms this. The performance of the local model on the global distribution depends on the quantity and similarity of the data, making it challenging to achieve a fair FL model through local debiasing alone.

To address this problem, in AFed, a conditional generator/GAN is trained collectively by all clients to extract knowledge about the local data distribution. The result is illustrated in  Fig. \ref{fig: toy example}. It can be seen that the generated data has a similar distribution to that of all client data, indicating that the generator was successful in extracting the knowledge about the local data distribution. The extracted knowledge can be shared with all clients to help local debiasing. With the generated sample, clients can access the global data distribution and achieve a more accurate debiasing outcome. Fig. \ref{fig: FedFair} illustrates the approach.

\begin{figure}[htp] 
    \centering
    \begin{subfigure}[b]{0.24\textwidth}
        \includegraphics[width=\textwidth]{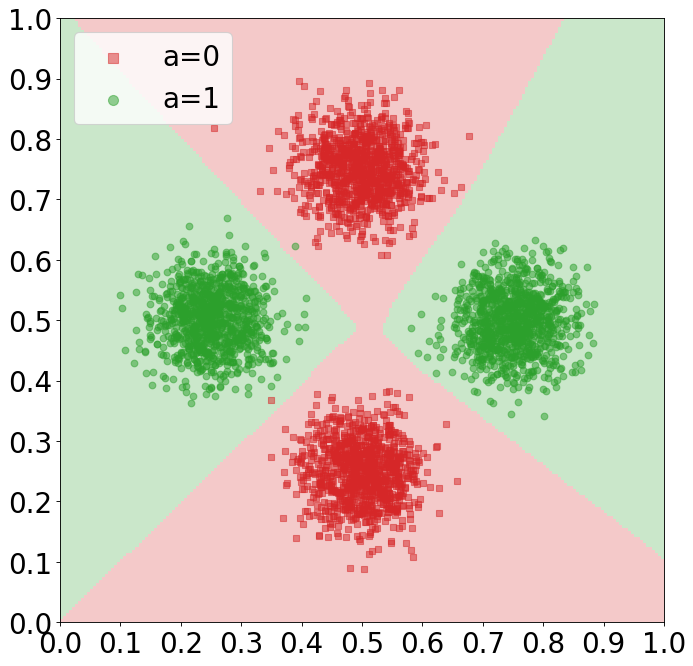} 
        \caption{}
        \label{fig:toy example all client1 data}
    \end{subfigure}
    \hfill
    \begin{subfigure}[b]{0.24\textwidth}
        \includegraphics[width=\textwidth]{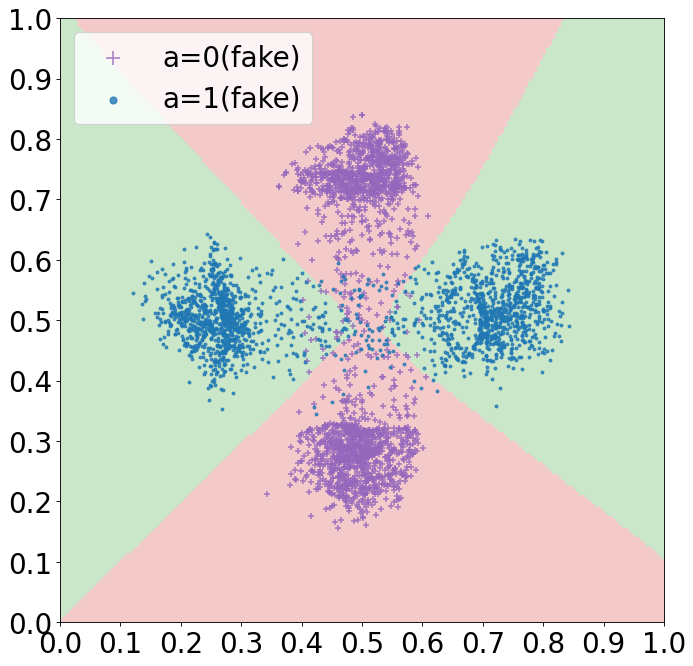}
        \caption{}
        \label{fig: toy example fake data}
    \end{subfigure}
    \caption{(a) The global distribution of all clients' data; (b) The distribution of generated data, the generated data is distributed in the same space as the real data.}
    \label{fig: toy example}
\end{figure}
\subsection{Overview of AFed}
This section details the proposed AFed framework, summarized in Algorithm \ref{alg: AFed}. Fig.\ref{fig: FedFair} shows the general description of AFed, which involves two learning tasks.  As shown on the left part of Fig.\ref{fig: FedFair}, the first task is to extract knowledge about local distribution. To this end, we train a conditional generator/GAN to learn clients' local data distribution. We then aggregate the extracted distributions to form a global distribution embedded in a conditional generator $G$. 

In the right part of Fig.\ref{fig: FedFair}, our second task is to train a fair and accurate model $f$, consisting of a feature extractor and classification head. The accuracy loss $\mathcal{L}_{\textrm{acc}}$ is obtained on real samples. To debias the model, the generator produces a corresponding fake latent feature with a different attribute value for each latent feature. We mix the true features with the fake ones and require the model to make consistent predictions on the mixed samples, which gives the fairness loss $\mathcal{L}_{\textrm{fair}}$. Each of these two tasks is discussed in more detail in the following.
 
\begin{figure*}[t]
\includegraphics[width=\textwidth]{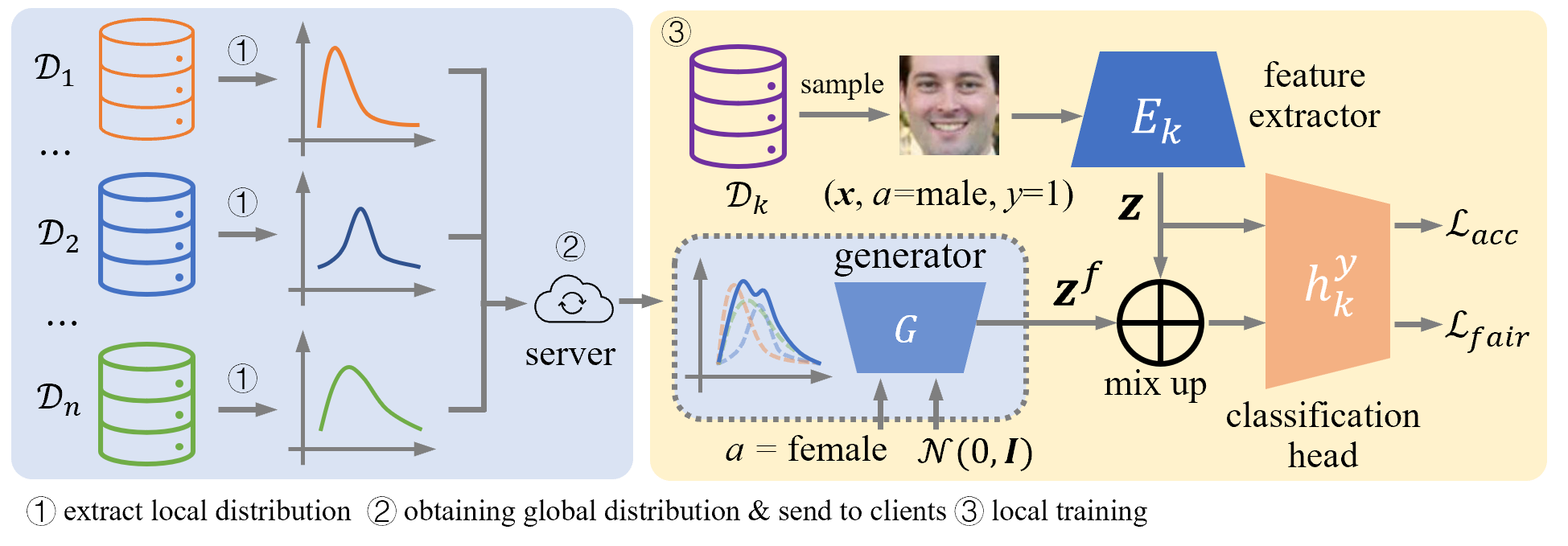}
\caption{An overview of \textbf{AFed Framework}. The key idea is first to extract clients' local distributions and obtain a view of global distribution via a conditional generator $G$, which is then shared with all clients to help local debiasing.}
\label{fig: FedFair}
\end{figure*}
\subsection{Learning Global Data Distribution}
Our first task is to learn the global data distribution. We draw inspiration from the field of knowledge extraction \cite{zhu2021data, NEURIPS2019_596f713f}. We extract local data distribution information before aggregating it to form a global distribution. We chose to perform this task in the latent space $\mathcal Z \in \mathbb{R}^p$ instead of input space $\mathcal X \in\mathbb{R}^d$ for both privacy and practical concerns. Because 1) sample-specific information is compressed in the latent space, only class-related information is retained; and 2) the latent space $\mathcal Z $ is more compact than the input space $\mathcal X$, making optimization more manageable. Note that the object of this paper is to train an algorithmic fair FL model without accessing client data, which is understudied in the literature. We leave both private and fair FL to our future work.

The feature extractor $E$ trained with $h^y$ will only retain the information relevant to $y$. While information about the attribute $a$ may be compressed. To address this, we designed an auxiliary classification head $h^a$ that can assist $E$ in extracting informative latent features. The benefit of $h^a$ is twofold: first, it provides feedback to the extractor $E$, allowing it to extract more informative latent features; second, $h^a$ embeds the relationships between the latent feature $z$ and the sensitive attribute $a$, i.e., $P(A|Z)$. This knowledge is crucial for the server to learn a global data distribution in AFed-G.

For the $k$-th client, the training process of the feature extractor $E(\cdot; \boldsymbol{\theta}^E_k)$ and the classification head $h^a(\cdot; \boldsymbol{\theta}^a_k)$ is formulated as follows:
\begin{equation} \label{eq: opt ha}
\min_{\boldsymbol{\theta}^a_k, \boldsymbol{\theta}^E_k}J_1(\boldsymbol{\theta}^a_k,\boldsymbol{\theta}^E_k) \coloneqq \frac{1}{|\mathcal{D}_k|}\sum_{\mathcal{D}_k} \ell_{CE} \left[ h^a(E(x_i;\boldsymbol{\theta}^E_k);\boldsymbol{\theta}^a_k)),a_i\right],
\end{equation}
where $\ell_{CE}$ is the cross-entropy loss.

A typical federated learning (FL) system faces the challenge of non-i.i.d. data, where each client possesses a distinct data distribution. As a result, a model that is fair in the local data may be unfair when applied to data sampled from the global data distribution. To address this issue, we propose augmenting local data with generated samples. The key is to capture the global data distribution without violating the privacy constraints of the FL system. To achieve this, we introduce two approaches.

\begin{algorithm} [tb]
    \caption{AFed(\textcolor{red}{red: AFed-G}, \textcolor{blue} {blue: AFed-GAN})} 
    \label{alg: AFed}
    \textbf{Input}: Local dataset $\left\{\mathcal{D}_k\right\}_{k=1}^N$.\\
    \textbf{Parameter}: learning rate $\beta$, and $\eta$, epoch $T_1$, $T_2$,  $E$.\\
    \textbf{Output}: $\boldsymbol{\theta}^E$, $\boldsymbol{\theta}^y$,$\boldsymbol{\theta}^a$
    \begin{algorithmic}[1]  
        \For{t=1,2,...,$E$}
            \State Randomly selects a subset of clients $\mathcal{S}$. 
            \For{each client $k \in \mathcal{S}$ in parallel}
                \State $\boldsymbol{\theta}_k^f, \boldsymbol{\theta}_k^a,$\textcolor{blue}{$\boldsymbol{w}_k^G,\boldsymbol{w}_k^D$} $\gets$ \Call{Client}{$\boldsymbol{\theta}^f, \boldsymbol{\theta}^a$, $\boldsymbol{w}^G$,
                \textcolor{blue}{$\boldsymbol{w}^D$}}.
            \EndFor
            \State Aggregates $\boldsymbol{\theta}_k^f, \boldsymbol{\theta}_k^a,$\textcolor{blue}{$\boldsymbol{w}_k^G,\boldsymbol{w}_k^D$} 
            \State \textcolor{red}{Updates $G$: $\boldsymbol{w}^G \gets \boldsymbol{w}^G - \alpha \nabla J_2(\boldsymbol{w}^G)$} 
        \EndFor
        \Function{Client}{$\boldsymbol{\theta}^f, \boldsymbol{\theta}^a$, $\boldsymbol{w}^G$, \textcolor{blue}{$\boldsymbol{w}^D$}}
            \State $\boldsymbol{\theta}^f_k, \boldsymbol{\theta}^a_k$, $\boldsymbol{w}_k^G$,\textcolor{blue}{$\boldsymbol{w}_k^D$} $\gets \boldsymbol{\theta}^f, \boldsymbol{\theta}^a$,$\boldsymbol{w}^G$, \textcolor{blue}{$\boldsymbol{w}^D$}
            \For {i=1,2,...,$T_1$} 
                \State Update $E$: $\boldsymbol{\theta}_k^E \gets \boldsymbol{\theta}_k^E - \eta \nabla J_1(\boldsymbol{\theta}_k^E)-\beta \nabla J_5(\boldsymbol{\theta}_k^E)$
                \State Update $h^a$: $\boldsymbol{\theta}_k^a \gets \boldsymbol{\theta}_k^a - \eta \nabla J_1(\boldsymbol{\theta}_k^a)$ 
                \State Update $h^y$: $\boldsymbol{\theta}_k^y \gets \boldsymbol{\theta}_k^y - \beta \nabla J_5(\boldsymbol{\theta}_k^y)$  
            \EndFor
            \textcolor{blue}{\For {i=1,2,...,$T_2$}
                \State Update D: $\boldsymbol{w}_k^D \gets \boldsymbol{w}_k^D$ - $\gamma_D \nabla J_3(\boldsymbol{w}_k^D)$
                \State Update G: $\boldsymbol{w}_k^G \gets \boldsymbol{w}_k^G$ - $\gamma_G \nabla J_4(\boldsymbol{w}_k^G)$
            \EndFor}
            \State \Return {$\boldsymbol{\theta}^f_k,\boldsymbol{\theta}^a_k$},\textcolor{blue}{$\boldsymbol{w}_k^G, \boldsymbol{w}_k^D$}
        \EndFunction
    \end{algorithmic}
\end{algorithm}

\subsubsection{Conditional generator approach}
We implement our first method AFed-G based on a conditional generator trained on the server side. Denote the ground-truth global joint distribution $p(Z,A)$, our target is to learn the global conditional distribution $p(Z|A)$, where $Z=E(X)$ represents the latent feature. We approximate $p(Z|A)$ by finding a distribution $Q^*$ that satisfies
\begin{equation} \label{eq: ori_obj}
        Q^*=\operatorname*{arg\,max}_{Q} \mathbb{E}_{a\sim p(a),z\sim Q(z|a)}\left[\log p(a|z)\right], 
\end{equation}
where $p(a)$ and $p(a|z)$ are the ground-truth prior and posterior distribution of the sensitive attribute $a$, respectively. The target of Eq. \ref{eq: ori_obj} is to find a distribution $Q(z|a)$ from which we can sample a latent feature $z$ conditional on the given value $a$. The sampled $z$ should be correlated with the sensitive attribute $a$, \textit{i.e.}, maximizing $p(a|z)$.

However, $p(a|z)$ is unknown in FL. We estimated it by the sum of local distributions that are embedded in clients' classification head $h^a_k(\cdot; \boldsymbol{\theta}^{a}_k)$ \cite{zhu2021data}:
\begin{equation} \label{eq: dist_sum}
    p(a|z) \propto \sum^N_{k=1} \log p(a|z;\boldsymbol{\theta}_k^{a}).
\end{equation}
The uploaded models can provide the server with a global view of the data distribution without violating the privacy requirement of FL. Based on the relation of Eq. \ref{eq: dist_sum}, we approximate Eq. \ref{eq: ori_obj} as:
\begin{equation} \label{eq: AFed-G-opt}
        Q^*=\operatorname*{arg\,max}_{Q} \mathbb{E}_{a\sim p(a),z\sim Q(z|a)}\left[\sum^N_{k=1} \log p(a|z;\boldsymbol{\theta}_k^{a})\right].
\end{equation}

Solving Eq. \ref{eq: AFed-G-opt} is difficult, instead, we train a conditional generator $G(\cdot;\boldsymbol{w}^G)$ to learn the conditional distribution $Q(Z|A)$ to approximate the solution. The training objective is formulated as follows:
\begin{equation} \label{eq: optG}
    \mathop{\min \limits_{\boldsymbol{w}^G}}J_2(\boldsymbol{w}^G) \coloneqq \mathbb{E}_{a\sim Q(a),z\sim G(z|a;\boldsymbol{w}^G)}\sum^N_{k=1}\ell_{CE}\left[h^a_k(z;\boldsymbol{\theta}_k^{a}),a\right],
\end{equation}
The latent feature $\boldsymbol{z}$ generated by the optimal generator should be correctly classified as the pre-defined attribute $a$ by all clients' classification heads $h^a_k(\cdot; \boldsymbol{\theta}^{a}_k)$. This way, the knowledge about the local data distribution is distilled to $G(\cdot;\boldsymbol{w}^G)$ and forms a global data distribution.

The conditional generator has demonstrated the ability to extract knowledge about the local data distribution in FL \cite{zhu2021data}. However, a conditional generator trained with the classification feedback extracts knowledge about the decision boundary. This knowledge does not adequately describe the local data distribution of clients, especially those far from the decision boundary. Thus a conditional generator may fail to provide gains for those samples.

\subsubsection{Conditional GAN approach}
In our second method, we replace the conditional generator with a conditional GAN parameterized by $\boldsymbol{w} = \boldsymbol{w}^D \circ \boldsymbol{w}^G$. Compared to the first method, this incurs computation overhead but can receive more informative feedback from the discriminator. We train the generator $G$ and discriminator $D$ on the client side and then aggregate them on the server side. For the $k$-th client, the training objective function of $D$ is formed as:
\begin{equation}\label{eq: GAN_D}
    \begin{split}    \min_{\boldsymbol{w}^{D}_k}J_3(\boldsymbol{w}^{D}_k)\coloneqq\mathbb{E}_{a,z\sim P_k(z,a)} \left[\log(1-D_k(z;\boldsymbol{w}^{D}_k))\right. \\     +\left.\log(D_k(G_k(\epsilon,a;\boldsymbol{w}^{G}_k);\boldsymbol{w}^{D}_k)) \right],
    \end{split}
\end{equation}
where $p_k(z|a)$ is the target distribution, $\boldsymbol{w}^{G}_k$ and $\boldsymbol{w}^{D}_k$ are training parameters of the generator and discriminator, respectively. The objective function of $G$ is formed as:
\begin{equation} \label{eq: GAN_G}
    \begin{split}   \min_{\boldsymbol{w}^{G}_k}J_4(\boldsymbol{w}^{G}_k)\coloneqq\mathbb{E}_{a,z\sim G(z|a;\boldsymbol{w}^{G}_k)} \left[\log(1-D_k(z;\boldsymbol{w}^{D}_k) \right].   
    \end{split}
\end{equation}
After local training, both $\boldsymbol{w}^G_k$ and $\boldsymbol{w}^D_k$ are uploaded to the server for aggregation. Algorithm \ref{alg: AFed} describes the detailed procedure.
\subsection{Local Debiasing}
The second task is performed on the client side. It aims to train a fair and accurate classifier $f$ with the help of $G$. The classification model $f$ is trained on both true and mixed samples, and the objective function is defined as
\begin{equation}  \label{eq: base opt problem}
    \min_{\boldsymbol{\theta}^f_k} J(\boldsymbol{\theta}^f_k) \coloneqq  {\mathcal{L}}^{y}_k(\boldsymbol{\theta}^f_k) + \lambda  {\mathcal{L}}^{fair}_k(\boldsymbol{\theta}^f_k),
\end{equation}
where ${\mathcal{L}}^{y}_k(\boldsymbol{\theta}^f_k)$ and ${\mathcal{L}}^{fair}_k(\boldsymbol{\theta}^f_k)$ are model's classification loss and fairness violation loss, respectively. $\lambda$ controls the trade-off between the model's accuracy and fairness. A larger $\lambda$ indicates a greater emphasis on fairness, while a smaller $\lambda$ improves the model's overall accuracy at the expense of fairness. The empirical loss $\hat{\mathcal{L}}^{y}_k(\boldsymbol{\theta}^f_k)$ is obtained on  $\mathcal{D}_k$,
\begin{equation} \label{eq: opt f}
    \hat {\mathcal{L}}^{y}_k(\boldsymbol{\theta}^f_k) \coloneqq \frac{1}{|\mathcal{D}_k|}\sum_{\mathcal{D}_k}\ell_{CE}\left[ f(x_i;\boldsymbol{\theta}^f_k),y_i\right],
\end{equation} 
We focus on DP \cite{hardt2016equality} in this work, which requires the prediction $\hat Y$ to be independent of the sensitive attribute $A$. The straightforward idea is to treat the discrimination level as a regularizer \cite{madras2018learning}, i.e., ${\mathcal{L}}^{fair}_k(\boldsymbol{\theta}^f_k)\coloneqq\lambda \Delta DP(f)$. 

However, in FL, the data is partitioned among all clients, with each holding only a subset of the total. Debiasing with only local data will fail in global data distribution. Thus, we propose augmenting the local dataset with synthetic data. More specifically, for a real sample $(\boldsymbol{x}_t,y_t,a_0)$ from the training dataset $\mathcal{D}_k$, the corresponding latent feature is $z_t=E(\boldsymbol{x}_t)$, where $E$ is the feature extractor. We generate a fake latent feature with a different sensitive attribute $a_1$:
\begin{equation} \label{eq: genZ}
    z_f = G(\epsilon, a_1;\boldsymbol{w}^G),
\end{equation}
where $\epsilon \sim N(0, \boldsymbol{I})$ is the noise input. We mix the true feature with the fake one to get an augmented sample \cite{zhang2018mixup}:
\begin{equation}
    \bar {z} = t z_t + (1-t){z}_f,
\end{equation}
where $t \sim \textrm{Beta}(\alpha,\alpha)$, for $\alpha \in (0, \infty)$. 

One assumption made in the mixup is that feature linear interpolations can be translated into label linear interpolations. Incorporating this kind of prior knowledge implicitly helps to broaden the model's training data distribution and, in turn, its ability to make accurate predictions on new data. Similarly, the sensitive attribute of $\bar {z}$ can be assumed as $\bar a = t a_0 +(1-t) a_1$. To ensure fairness, the model's prediction should be invariant to the interpolation degree $t$ and consequently the intermediate value of sensitive attribute \cite{chuang2020fair}:
\begin{equation} \label{eq: fair constraint}
  \frac{d}{dt} \mathbb{E}_{\boldsymbol{z}_t,\boldsymbol{z}_f}\left[ f\left(\bar {z}\right)\right] \equiv 0.
\end{equation}
The penalized optimization problem is now formulated as follows:
\begin{equation} \label{eq: optf}
 \min_{\boldsymbol{\theta}^f_k} J_5(\boldsymbol{\theta}^f_k) \coloneqq  {\mathcal{L}}^{y}_k(\boldsymbol{\theta}^f_k) + \lambda\mathbb{E}_{z_t,z_f}\left[\langle \nabla_z f(\bar{z}), z_t-z_f \rangle\right].
\end{equation}
By minimizing Eq. \ref{eq: optf}, we are optimizing the model $\boldsymbol{\theta}^f_k$ to, on the one hand, fit the local data distribution of dataset $\mathcal{D}_k$, and on the other hand, regularize the model's disparate predictions on the mixed distribution of $z^t$ and $z^f$.

\begin{figure*}
    \begin{subfigure}[b]{\columnwidth}
        \includegraphics[width=\textwidth]{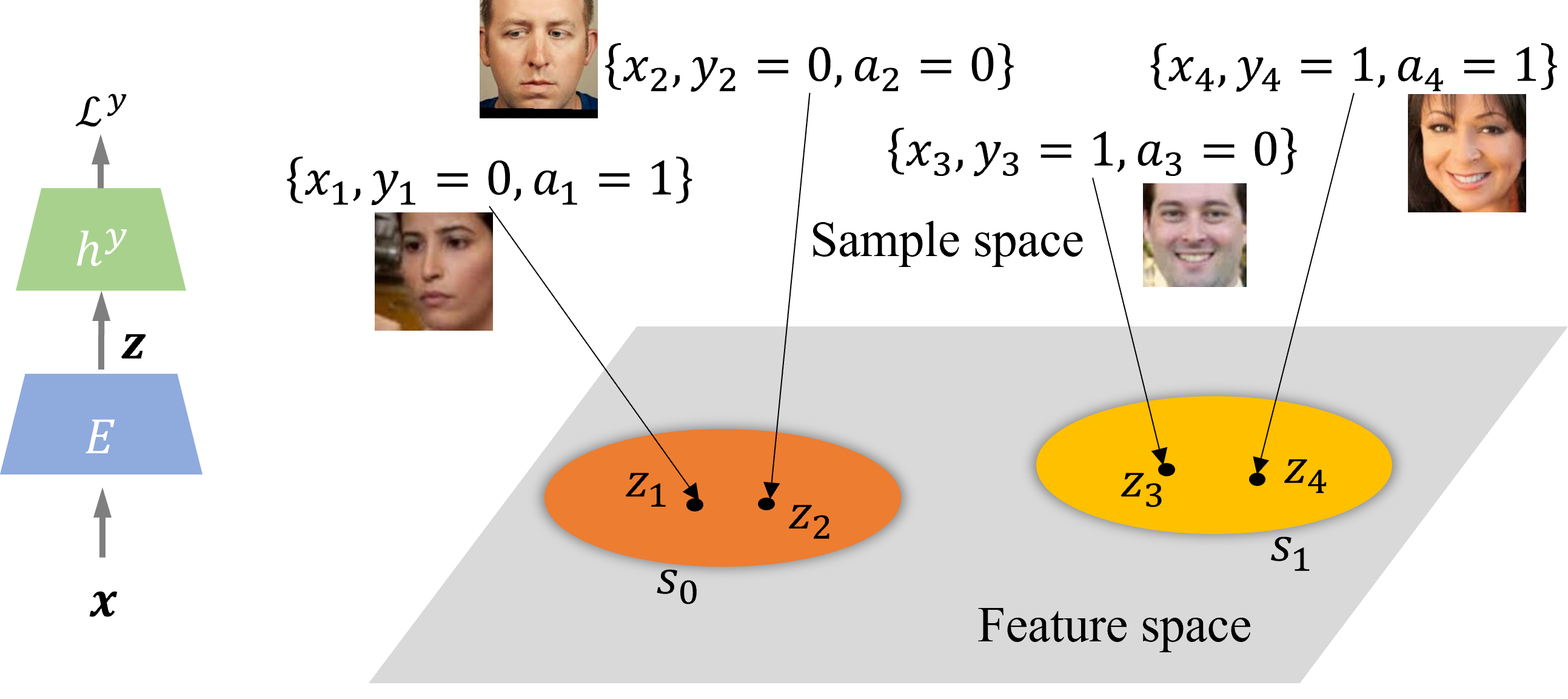}
        \caption{Training with one classification head}
        \label{fig: tsch}
    \end{subfigure}  
    \hfill
    \begin{subfigure}[b]{\columnwidth}
        \includegraphics[width=\textwidth]{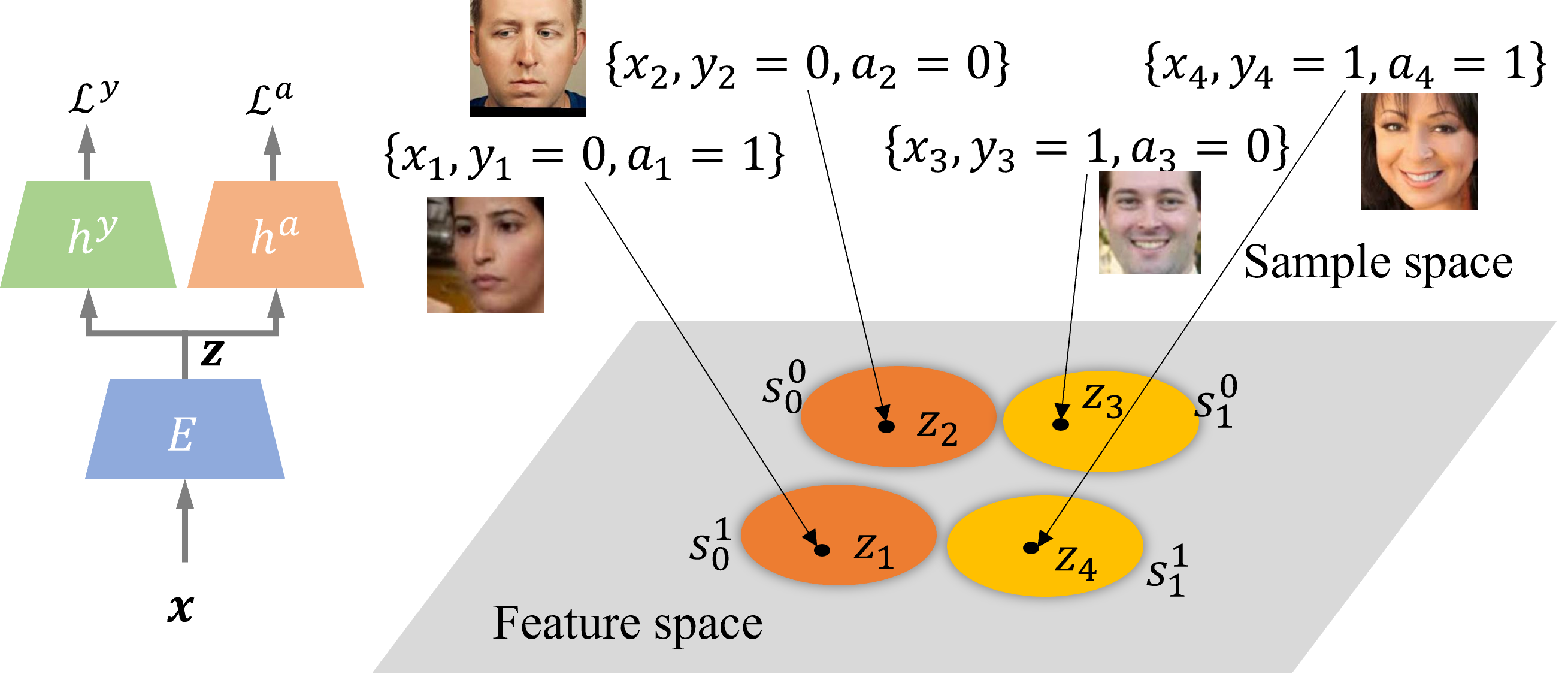}
        \caption{Training with two classification heads}
        \label{fig: ttch}
    \end{subfigure} 
    \caption{(a) One classification head case, the extractor $E_k$ is trained solely with the feedback of $h^y_k$. Samples with the same label $y=1$ but different attributes $a=0$ and $a=1$ are mapped to the same area in latent space. (b) Two classification heads case, $E_k$ is now trained with both heads' feedback. Samples are better separated by $y$ and $a$.}
\end{figure*}

\section{Analysis}
This section presents our analysis. We begin by showing $h^a$ is necessary for the extractor to learn informative features. Then, we analyze the drawbacks of AFed-G. In the third part of this section, we thoroughly analyzed the fairness guarantee of our framework. First, we demonstrate that debiasing with local data is futile. Next, we show that generated samples improve the local model's fairness performance on global data distribution. Since Rasou et al. \cite{rasouli2020fedgan} have proven the convergence of GAN training in FL with non-i.i.d. data, we omit the convergence analysis of AFed-GAN.

\subsection{Feature Extractor Analysis} 
Fig. \ref{fig: tsch} shows the one classification head case. The resultant extractor $E$ tends only to retain information pertinent to label $y$ and compresses all other information. $\{x_1,y_1=0,a_1=1\}$ and $\{x_2,y_2=0,a_2=0\}$ are mapped to the same area due to the same label $y=1$ despite having different attribute values $a$. This brings confusion when training conditional GAN in this feature space. 

When training conditional GAN, the generated sample should resemble the training data. In other words, the fake samples should be close to the real ones in the feature space. Near $z_2$, within the area $s_0$, there are samples with attribute values $a=0$ and $a=1$. As a result, feeding the generator with attribute value $a=0$ will produce samples similar to group $a=0$ or group $a=1$, as depicted in Fig. \ref{fig: UMAP celeba tsch}. The same holds for the group $a=1$. 

To address this, we propose adding an auxiliary classification head $h^a$ to provide more feedback for the extractor as shown in Fig.\ref{fig: ttch}. With the guidance of $h_y$ and $h^a$, the extracted features contain information about $y$ and $a$. Samples from groups defined by $y$ and $a$ are mapped to four distinct feature space regions. On this basis, we can train a generator capable of generating samples conditional on $a$.

Fig. \ref{fig: UMAP of celeba} shows the results of the single and two classification head cases. The inputs to $E$ are sampled from the CelebA dataset. As shown in Fig. \ref{fig: UMAP celeba tsch}, the generated data fails to match the real data. For $a=0$, the generated features (purple cross) overlap with the real features with attributes $a=1$ and $a=0$. As a comparison, Fig. \ref{fig: UMAP celeba ttch} demonstrates that the auxiliary feedback signal from $h^a$ significantly improves generator performance. The extracted features are divided into four parts, corresponding to the four groups defined by $y$ and $a$. Since the features are well-separated, the GAN can learn the maps from $a$ to $\boldsymbol{z}$ flawlessly.  

\begin{figure} 
    \centering
    \begin{subfigure}[b]{0.2455\textwidth}
        \includegraphics[width=\textwidth]{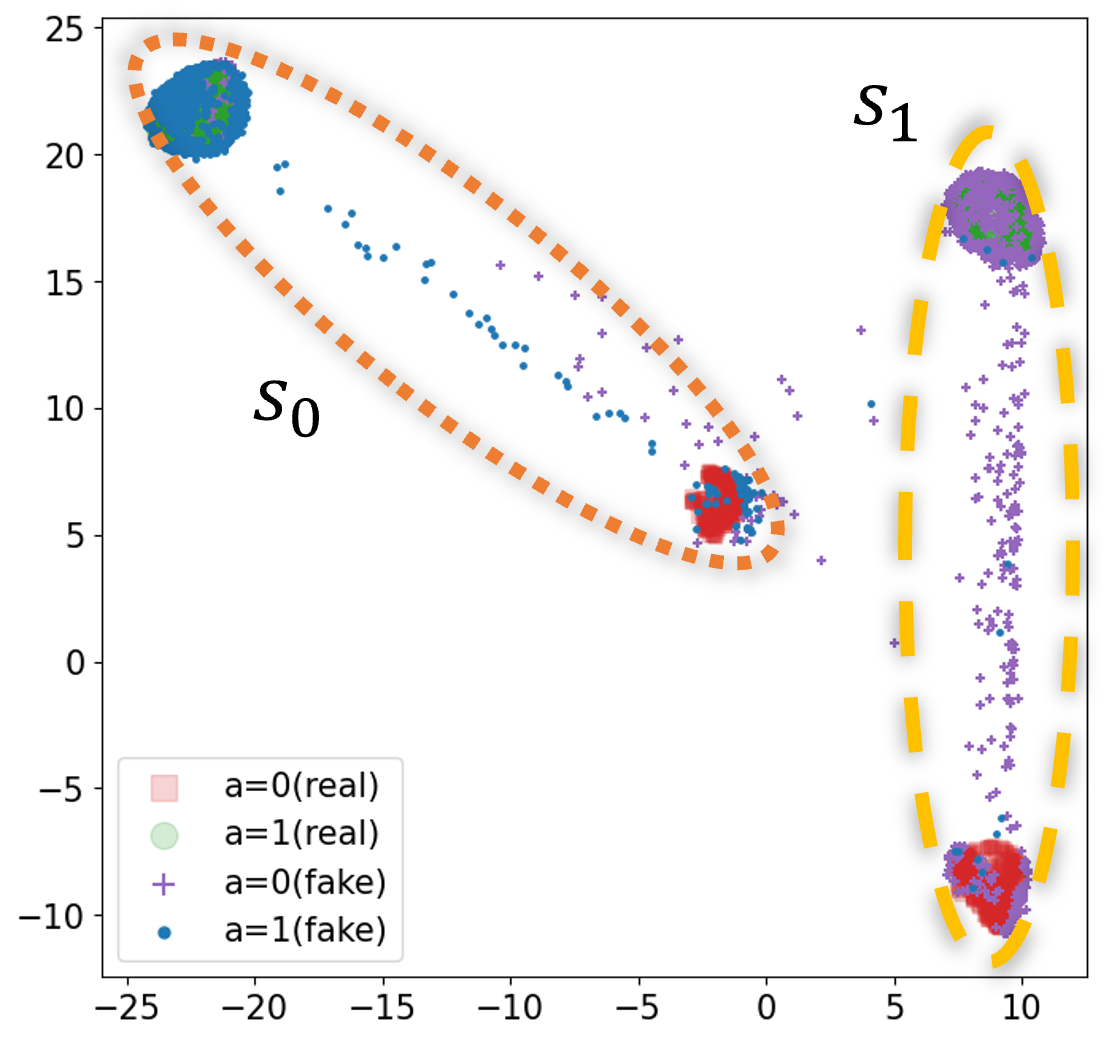}
        \caption{}
        \label{fig: UMAP celeba tsch}
    \end{subfigure}
    \hfill
    \begin{subfigure}[b]{0.2345\textwidth}
         \includegraphics[width=\textwidth]{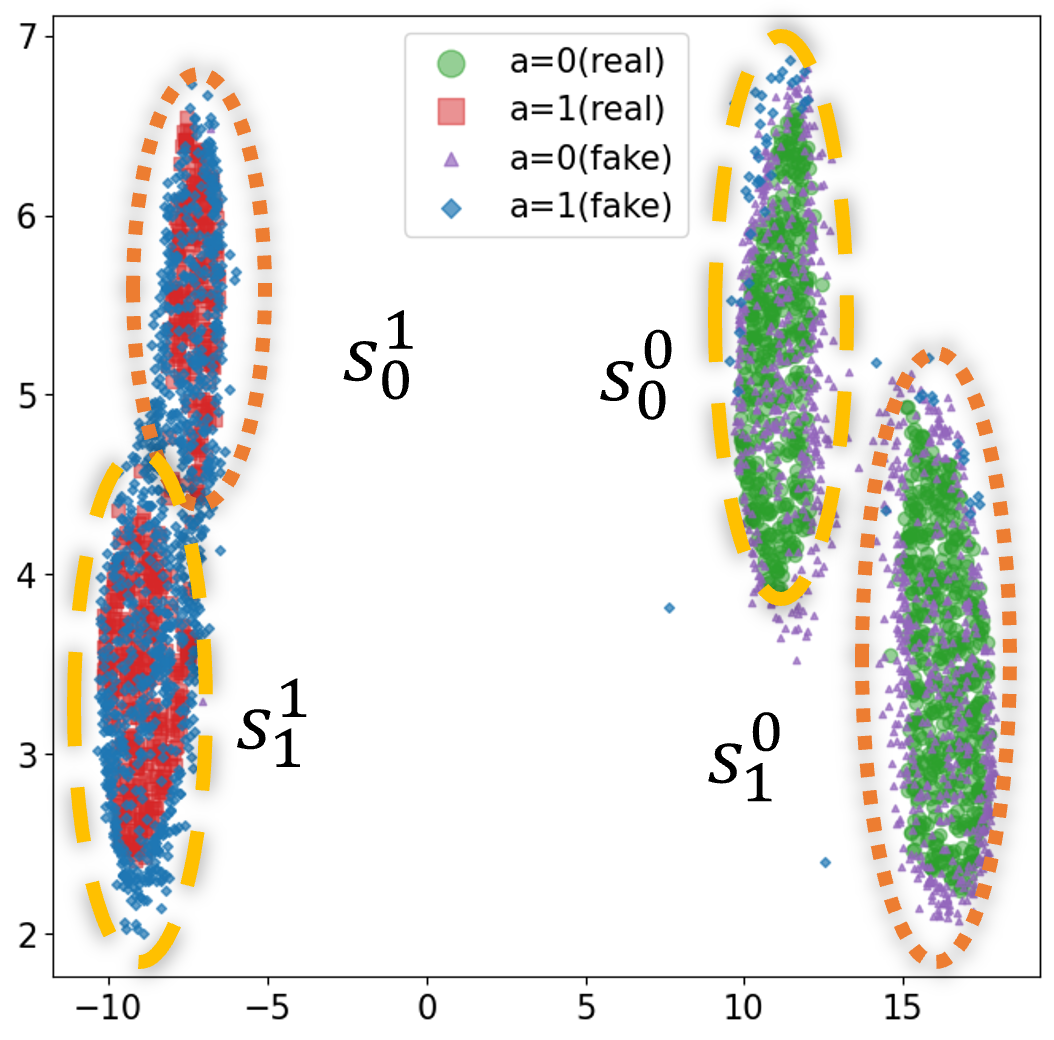}
        \caption{}
        \label{fig: UMAP celeba ttch}   
    \end{subfigure}
    \caption{The UMAP \cite{mcinnes2018umap} of fake and real features. (a) One classification head case; (b) Two classification heads case.}
    \label{fig: UMAP of celeba} 
\end{figure}

\subsection{Generator Analysis}
Our first method AFed-G trains a conditional generator $G$ to approximate the $P(Z|A)$ using the knowledge embedded in the classification head $\{\boldsymbol{\theta}_k^a\}_{k=1}^{|S|}$, i.e. $p_k(a|z)$. The second method, AFed-GAN, is more direct. A discriminator guides training to train a conditional GAN to learn the distribution $p_k(z|a)$. The former enjoys simplicity (only a generator is added) and can be carried out on the server's side, which reduces the client's training and communication load. In comparison, the latter performs better in learning the local data distribution.  

Since introduced in \cite{goodfellow2014generative}, numerous researchers have successfully implemented GAN to learn the target distribution \cite{mao2019mode, mirza2014conditional}. In the context of FL, \cite{rasouli2020fedgan} demonstrates the convergence of GAN training in FL with non-i.i.d. data. Here, we focus on analyzing the shortcomings of the AFed-G approach.
\begin{theorem}
    Given the optimal classification head $h^y$, optimizing Eq. \ref{eq: AFed-G-opt} is equivalent to
    \begin{equation}
        \min_Q \textrm{KL}\left(Q_{A,Z}\|\prod\limits_{k=1}^N{P^k_{A,Z}}\right) - \textrm{KL}\left(Q_{Z}\|\prod\limits_{k=1}^N{P^k_{Z}}\right)+H_Q(A\|Z) 
    \end{equation}
    where $\prod\limits_{k=1}^N{P^k_{A,Z}}$ and $\prod\limits_{k=1}^N{P^k_{Z}}$ are the joint probability distribution of clients' data. $H_Q(A\|Z)=\mathbb{E}_{a,z\sim Q_{A,Z}}[\log q(a|z)]$ is the conditional entropy of the generated samples.
\end{theorem}
\begin{proof}
Let's denote $p_k(a|z) \coloneqq p(a|z;\boldsymbol{\theta}_k^a)$ the likelihood of client $k$'s local data distribution, which is embedded in the classification head $\boldsymbol{\theta}_k^a$. Eq. \ref{eq: AFed-G-opt} aims to find the target distribution $Q^*$. Note we lose $\boldsymbol{\theta}_k^a$ for simplicity.
\begin{align*}
    & \max_Q \mathbb{E}_{a\sim Q(a)}\mathbb{E}_{z\sim Q(z|a)}\sum_{k=1}^N{\log p_k(a|z)} \\
    = & \max_Q \mathbb{E}_{a\sim Q(a)}\mathbb{E}_{z\sim Q(z|a)}\left[\sum_{k=1}^N{\left[\log \frac{p_k(a,z)}{p_k(z)}\right] }  + \log \frac{q(a,z)}{q(a,z)} \right] \\
    = & \min_Q \mathbb{E}_{a\sim Q(a)}\mathbb{E}_{z\sim Q(z|a)}\left[\log q(a,z) - \log{\prod\limits_{k=1}^N{ p_k(a,z)}} \right.\\
     & - \left. \log q(z) + \log{\prod\limits_{k=1}^N{p_k(z)}} -\log q(a|z)  \right] \\
    = & \min_Q \textrm{KL}\left(Q_{A,Z}\|\prod\limits_{k=1}^N{P^k_{A,Z}}\right) - \textrm{KL}\left(Q_{Z}\|\prod\limits_{k=1}^N{P^k_{Z}}\right) \\
    & + H_Q(A\|Z). 
\end{align*}
\end{proof}

Optimizing Eq. \ref{eq: AFed-G-opt} comprises three terms. The first minimizes the KL divergence between the conditional generator $G$ distribution and the ground-truth joint distribution of $N$ clients, which is our target, matching the $Q_{A,Z}$ with $\prod\limits_{k=1}^N{P^k_{A,Z}}$. However, the second term maximizes the KL divergence between the marginal generator $Q_{Z}$ and target joint distributions $\prod\limits_{k=1}^N{P^k_{Z}}$. This contradicts our goal and will degrade the generator's performance by pushing the marginal distribution away from the ground truth distribution. The last term is the conditional entropy of label $A$ of data sampled from the generator distribution $Q$. Minimizing the last term will decrease the data's diversity by making generated $\boldsymbol{z}$ be classified as $A$ with higher confidence, i.e., away from the decision boundary. This is in accordance with our intuition that only the feedback from classification head $\boldsymbol{\theta}_k^a$ is available. This is to say, only information related to the decision boundary is provided, and the model will learn to generate "safe" examples that can be correctly classified by $\boldsymbol{\theta}_k^a$. As a result, the generated data may fail to distribute in the space spanned by real data. 

\subsection{Fairness Analysis}
\subsubsection{Debiasing with local data is ineffective}
Our analysis borrows the idea from the field of domain adaptation. Each client's data is associated with a source domain $\mathcal{D}_k$ from which local data is sampled. The goal is to train a model that performs well on the global (or target) domain, which is the ensemble of all source domains $\mathcal{D}_g = \sum_{k=1}^N \mathcal{D}_k$. 

We first introduce a metric to measure the local and global distribution differences. We use the $\mathcal{H}$-divergence  $d_{\mathcal{H}\Delta \mathcal{H}}$ \cite{ben2010theory} to measure the difference between two probabilities,
\begin{equation}
    d_{\mathcal{H}\Delta \mathcal{H}}(\mathcal{D}_k, \mathcal{D}_g)=2\sup_{h\in \mathcal{H}}\left|\Pr_{\mathcal{D}_k}[I(h)]-\Pr_{\mathcal{D}_g}[I(h)]\right|,
\end{equation}
where $\mathcal{H}$ is the hypothesis space, $I(h)$ is a set of samples $x$ which is defined as $I(h)=\left\{x|h(x)=1\right\}$. 
\begin{lemma}\label{lemma: ben}\cite{ben2006analysis}
    Let $\mathcal{H}$ be a hypothesis space of VC dimension $d$, $\mathcal{U}$ and $\mathcal{U}'$ are samples of size $m$, sampled from $\mathcal{D}$ and $\mathcal{D}'$, respectively. $d_{\mathcal{H}\Delta \mathcal{H}}(\mathcal{D}, \mathcal{D}')$ is the $\mathcal{H}$-divergence between $\mathcal{D}$ and $\mathcal{D}'$ and $\hat d_{\mathcal{H}\Delta \mathcal{H}}(\mathcal{D}, \mathcal{D}') $ is the empirical distance computed on $\mathcal{U}$ and $\mathcal{U}'$. Then for any $\delta \in (0,1)$, with probability at least 1-$\delta$,
    \begin{equation} \label{eq: h divergence gap}
        d_{\mathcal{H}\Delta \mathcal{H}}(\mathcal{D}, \mathcal{D}') \leq \hat d_{\mathcal{H}\Delta \mathcal{H}}(\mathcal{D}, \mathcal{D}')+4\sqrt{\frac{d\log(2m)+\log(\frac{2}{\delta})}{m}}
    \end{equation}
\end{lemma} 
Now we are ready to bound the fairness violation that a hypothesis $h_k$ trained with the local model can have on global data, i.e., $\Delta \textrm{DP}_{\mathcal{D}}(h_k)$. We borrow the idea from the field of domain adaptation. Each client's data can be viewed as a source domain $\mathcal{D}_k$, and the goal is to train a global model that performs well on the global domain, which is the ensemble of all source domains $\mathcal{D}=\sum_{k=1}^{N}\mathcal{D}_k$.
\begin{theorem}\cite{schumann2019transfer} \label{theorem: fairness upper bound}
    Let $\mathcal{H}$ be a hypothesis space of VC dimension $d$ with a ground truth labeling function $f: \mathcal{D} \mapsto \mathcal{Y}$, $h_k\in \mathcal{H}$ is a local model trained on local dataset $\mathcal{D}_k$. Then for any $\delta \in (0,1)$, with probability at least 1-$\delta$ (over the choice of samples), the discrimination level of $h_k$ on $\mathcal{D}$ is bounded by
    \begin{equation} \label{eq: fairness generalization}
    \begin{split}
        \Delta \text{DP}_{\mathcal{D}_g}\left(h_k\right) \leq \frac{1}{2}d_{\mathcal{H}\Delta \mathcal{H}}\left(\mathcal{U}_g^0,\mathcal{U}_k^0\right) + \frac{1}{2}{d_{\mathcal{H}\Delta \mathcal{H}}}\left(\mathcal{U}_g^1,\mathcal{U}_k^1\right) + {\lambda^0_k} \\
        + {\lambda^1_k} + \Delta\text{DP}_{\mathcal{D}_k}\left(h_k\right) + 8\sqrt{\frac{2d\log (2m) + \log(2/d)}{m}},  
    \end{split}
    \end{equation}
    where $\mathcal{D}_g^a=\left\{X\sim \mathcal{D}_g| A=a\right\}$ is the subdomain defined by attribute $a$, similar as $\mathcal{D}_k^a$. $\lambda^a_k=\min_{h\in\mathcal{H}} \epsilon_{\mathcal{D}_g^a}(h_k,f)+\epsilon_{\mathcal{D}^a_k}\left(h_k,f \right)$ is the optimal risk a hypothesis $h_k$ can have on the union of domain $\mathcal{D}^a_g$ and $\mathcal{D}^a_k$. $\epsilon_{\mathcal{D}_g^a}(h_k,f)\coloneqq \mathbb{E}_{x\sim \mathcal{D}_g^a}[|h(x)- f(x)|]$ is the empirical loss of $h_k$ on $\mathcal{D}^a_g$, similar as $\epsilon_{\mathcal{D}_g^a}(h_k,f)$.
\end{theorem}
\begin{proof}
Our proof heavily uses the triangle inequality \cite{ben2006analysis},
\begin{equation}
    \epsilon(f_1,f_2) \le \epsilon(f_1,f_3) + \epsilon(f_2,f_3), 
\end{equation}
where $f_1,f_2,f_3$ are any labeling functions. Without loss of generality, we assume $\mathbb{E}_{x\sim \mathcal{D}^0}[h(x)] \ge \mathbb{E}_{x\sim \mathcal{D}^1}[h(x)]$. The discrimination level can be calculated as
\begin{align}
   \Delta \text{DP}_\mathcal{D}(h) &= \left|\mathbb{E}_{x\sim \mathcal{D}^0}[h(x)]-\mathbb{E}_{x\sim \mathcal{D}^1}[h(x)]\right| \\ \nonumber
   &= \mathbb{E}_{x\sim \mathcal{D}^0}[h(x)]+\mathbb{E}_{x\sim \mathcal{D}^1}[1-h(x)]-1 \\ \nonumber
   &= \epsilon_{\mathcal{D}^0}\left(h,f \right) + \epsilon _{\mathcal{D}^1}\left( {1 - h,f} \right) - 1
\end{align}
 The discrimination level of $h_k$ on $\mathcal{D}$ is
\begin{align*}
   &\Delta \text{DP}_\mathcal{D}\left(h_k\right) =  \epsilon_{\mathcal{D}^0}\left(h_k,f \right) + \epsilon _{\mathcal{D}^1}\left( {1 - {h_k},f} \right) - 1 \\
   &  \leq\epsilon _{{\mathcal{D}^0}}\left( h_k, h^* \right) + \epsilon _{{\mathcal{D}^0}}\left( f,h^* \right) + \epsilon_{\mathcal{D}^1}\left( {1-h_k,h^*} \right)   \\
   &\quad + \epsilon _{\mathcal{D}^1}\left(f,h^*\right) - 1   \\
   & = \epsilon _{\mathcal{D}^0}\left( h_k, h^* \right) - \epsilon _{\mathcal{D}_k^0}\left( h_k, h^* \right) + \epsilon _{\mathcal{D}_k^0}\left( h_k, h^* \right) + \epsilon _{\mathcal{D}^0}\left( f,h^* \right)   \\
   &\quad+ \epsilon _{\mathcal{D}^1}\left( {1 - {h_k},{h^*}} \right) - \epsilon _{\mathcal{D}_k^1}\left( {1 - {h_k},{h^*}} \right)   \\
   & \quad + \epsilon _{\mathcal{D}_k^1}\left( {1 - {h_k},{h^*}} \right) + \epsilon _{\mathcal{D}^1}\left( f,h^* \right) - 1   \\
   &  \leq\left| {\epsilon _{\mathcal{D}^0}\left( h_k, h^* \right) - \epsilon _{\mathcal{D}_k^0}}\left(h_k, h^* \right) \right|+ \epsilon _{\mathcal{D}_k^0}\left( h_k, h^* \right) + \epsilon _{\mathcal{D}^0}\left( f,h^* \right)  \\
   & \quad+ \left| {\epsilon _{\mathcal{D}^1}\left( {1 - {h_k},{h^*}} \right) - \epsilon _{\mathcal{D}_k^1}}\left( {1 - {h_k},{h^*}} \right) \right|   \\ 
   &\quad+ \epsilon _{\mathcal{D}_k^1}\left( {1 - {h_k},{h^*}} \right) + \epsilon _{\mathcal{D}^1}\left( f,h^* \right) - 1   \\
   &  \leq \frac{1}{2}{d_{\mathcal{H}\Delta \mathcal{H}}}\left( {{\mathcal{D}^0},\mathcal{D}_k^0} \right) + \frac{1}{2}{\mathcal{D}_{\mathcal{H}\Delta \mathcal{H}}}\left(\mathcal{D}^1,\mathcal{D}_k^1 \right) + \epsilon _{\mathcal{D}_k^0}\left( h_k, h^* \right)   \\
   &\quad+ \epsilon _{\mathcal{D}^0}\left( f,h^* \right) + \epsilon _{\mathcal{D}_k^1}\left( {1 - {h_k},{h^*}} \right) + \epsilon _{\mathcal{D}^1}\left( f,h^* \right) - 1   \\
   & \leq \frac{1}{2}{d_{\mathcal{H}\Delta \mathcal{H}}}\left( {{\mathcal{D}^0},D_k^0} \right) + \frac{1}{2}{d_{\mathcal{H}\Delta \mathcal{H}}}\left( {{\mathcal{D}^1},D_k^1} \right) + \epsilon _{D_k^0}\left( {{h_k},f} \right)   \\
   & \quad+ \epsilon _{D_k^0}\left( f,h^* \right) + \epsilon _{\mathcal{D}^0}\left( f,h^* \right) + \epsilon _{D_k^1}\left( {1 - {h_k},f} \right)   \\
   & \quad+ \epsilon _{D_k^1}\left( f,h^* \right) + \epsilon _{\mathcal{D}^1}\left( f,h^* \right) - 1   \\
   & = \frac{1}{2}{d_{\mathcal{H}\Delta \mathcal{H}}}\left( {{\mathcal{D}^0},D_k^0} \right) + \frac{1}{2}{d_{\mathcal{H}\Delta \mathcal{H}}}\left( {{\mathcal{D}^1},D_k^1} \right)   \\
   & \quad+ \epsilon _{D_k^0}\left( {{h_k},f} \right) + {\lambda _0} + \epsilon _{D_k^1}\left( {1 - {h_k},f} \right) + {\lambda _1} - 1   \\
   & = \frac{1}{2}{d_{\mathcal{H}\Delta \mathcal{H}}}\left( {{\mathcal{D}^0},D_k^0} \right) + \frac{1}{2}{d_{\mathcal{H}\Delta \mathcal{H}}}\left(\mathcal{D}^1,\mathcal{D}_k^1 \right) \\
   & \quad+ {\lambda _0} + {\lambda _1} + \Delta \textrm{DP}_{\mathcal{D}_k}\left(h_k\right) \\
   & \leq \frac{1}{2}d_{\mathcal{H}\Delta \mathcal{H}}\left(\mathcal{U}_g^0,\mathcal{U}_k^0\right) + \frac{1}{2}{d_{\mathcal{H}\Delta \mathcal{H}}}\left(\mathcal{U}_g^1,\mathcal{U}_k^1\right) + {\lambda^0_k} + {\lambda^1_k} + \\ & \quad \Delta\text{DP}_{\mathcal{D}_k}\left(h_k\right) + 8\sqrt{\frac{2d\log (2m) + \log(2/d)}{m}}.
\end{align*}
The last inequality is due to Lemma \ref{lemma: ben}.
\end{proof}
Theorem \ref{theorem: fairness upper bound} provides some insights on why local debiasing is ineffective. Consider a client $k$, who can implement some powerful local debiasing method that produces a perfectly fair hypothesis $h$ on the local domain $\mathcal{D}_k$. However, this fairness guarantee could fail to generalize to the global domain $\mathcal{D}_g$. Because the $\mathcal{H}$-divergence between $\mathcal{U}^a$ and $\mathcal{U}_k^a$ (a={0,1}) can be huge due to the non-iid fact in FL.
\subsubsection{Generated samples improve fairness on global distribution} The client can train a model with better fairness generalization performance on the global distribution with the generated samples.
\begin{remark} \label{lemma: G reduce d}
    Denote $\mathcal{D}'_k = (1-\alpha)\mathcal{D}_k + \alpha\mathcal{D}_g$ the client $k$'s data, which is the augmented by the generated samples drawn from $\mathcal{D}_g$. $\alpha$ is the ratio of generated samples over the total samples. The $\mathcal{H}$-divergence between local and global data. we have
    \begin{equation}
        d_{\mathcal{H}\Delta \mathcal{H}}\left(\mathcal{D},\mathcal{D}'_k\right) \le d_{\mathcal{H}\Delta \mathcal{H}}\left(\mathcal{D},\mathcal{D}_k\right).
    \end{equation}
\end{remark}
\begin{proof}
    According to the definition, we have
    \begin{align}
        & d_{\mathcal{H}\Delta \mathcal{H}}\left(\mathcal{D},\mathcal{D}'_k\right) \nonumber
        = 2\sup_{h\in \mathcal{H}}\left|\textrm{Pr}_\mathcal{D}[I(h)]-\textrm{Pr}_{\mathcal{D}_k'}[I(h)]\right| \\ \nonumber
        = & 2\sup_{h\in \mathcal{H}}\left|\textrm{Pr}_\mathcal{D}[I(h)]-(1-\alpha)\textrm{Pr}_{\mathcal{D}_k}[I(h)]-\alpha \textrm{Pr}_{\mathcal{D}_g}[I(h)]\right| \\\nonumber
        \leq & 2(1-\alpha)\sup_{h\in \mathcal{H}}\left|\textrm{Pr}_\mathcal{D}[I(h)]-\textrm{Pr}_{\mathcal{D}_k}[I(h)]\right| \\\nonumber
        & + 2\alpha \sup_{h\in \mathcal{H}}\left| \textrm{Pr}_\mathcal{D}[I(h)]- \textrm{Pr}_{\mathcal{D}_k}[I(h)]\right|\\ 
        = & d_{\mathcal{H}\Delta \mathcal{H}}\left(\mathcal{D},\mathcal{D}_k\right). 
    \end{align}
    The inequality is due to $d_{\mathcal{H}\Delta \mathcal{H}}\left(\mathcal{D},\mathcal{D}_g\right) \leq d_{\mathcal{H}\Delta\mathcal{H}}\left(\mathcal{D},\mathcal{D}_k\right)$. Because $G$ is trained to generate samples similar to the real samples drawn from the $\mathcal{D}$. 
\end{proof}
Remark \ref{lemma: G reduce d} shows augmenting the client's data with generated samples reduces the $\mathcal{H}$-divergence. The conclusion of Remark \ref{lemma: G reduce d} can be generalized to subdomains $\mathcal{D}^a$ ($a\in {0,1}$). In AFed, a conditional generator is designed to generate samples similar to those drawn from $\mathcal{D}^0$ and $\mathcal{D}^1$. Therefore, AFed is able to decrease $d_{\mathcal{H}\Delta \mathcal{H}}\left(\mathcal{D}^0,\mathcal{D}^0_k\right)$ and $d_{\mathcal{H}\Delta \mathcal{H}}\left(\mathcal{D}^1,\mathcal{D}^1_k\right)$ in Eq. \ref{eq: fairness generalization} and consequently induce a smaller discrimination level $\Delta \text{DP}_{\mathcal{D}_g}\left(h_k\right)$. In addition, a larger dataset size $m$ help close the gap between the empirical estimation and real $\mathcal{H}$-divergence $d_{\mathcal{H}\Delta \mathcal{H}}\left(\mathcal{D},\mathcal{D}_g\right)$ as shown in Eq. \ref{eq: h divergence gap}, which in turn helps improve the empirical performance.

\subsubsection{Mix-up samples improve fairness}
One reason for unfairness is the bias that exists in the unbalanced dataset. One way to alleviate this is to balance the dataset through data augmentation. As stated above, a conditional generator is collectively trained by all clients, which can generate fake representations conditional on sensitive value. We further proposed mixing the synthetic data with the real data to improve the quality of synthetic data. The sensitive attribute of the mixed data can be represented as $z=tz_0+(1-t)z_1$ \cite{zhang2018mixup}. To ensure the model makes consistent predictions over different demographic groups, we require the model's prediction on the mixed data to be invariant w.r.t $t$.

\begin{theorem} \cite{chuang2020fair} \label{fairmix}
Let $\bar {z} = t z_0 + (1-t)z_1$ be the mixed sample, $z_0$ and $z_1$ are samples with different sensitive attributes. For any differentiable function $f$, the discrimination level of $f$ on two groups $P_0$ and $P_1$ is
\begin{equation}
    \Delta \textrm{DP}(f) = \left|\int_0^1 \frac{d}{dt} \mathbb{E}_{\boldsymbol{z_0}\sim P_0,\boldsymbol{z_1}\sim P_1}\left[ f\left(\bar {z}\right)\right] dt\right|,
\end{equation}
\end{theorem} 

\begin{proof}
Let $T(z_0,z_1,t)=t z_0 + (1-t)z_1$, $t=0$ gives samples from the group $P_1$ and $t=1$ for group $P_0$ 
\begin{equation}
\begin{split}
    & \Delta \textrm{DP}(f) = |\mathbb{E}_{z_0\sim P_0}\left[f(x)\right] - \mathbb{E}_{z_1\sim P_1}\left[f(x)\right]| \\
    = & \left|\mathbb{E}_{z_0\sim P_0, z_1\sim P_1}\left[f(T(z_0,z_1,0))\right] \right.\\
    & - \left. \mathbb{E}_{z_0\sim P_0, z_1\sim P_1}\left[f(T(z_0,z_1,1))\right]\right| \\
    = &  \left|\mathbb{E}_{z_0\sim P_0, z_1\sim P_1}\left[f(T(z_0,z_1,0))-f(T(z_0,z_1,1))\right]\right| \\
    = & \left|\mathbb{E}_{z_0\sim P_0, z_1\sim P_1}\left[\int_0^1 \frac{d}{dt} f(T(z_0,z_1,t)) dt\right]\right| \\
    = & \left|\int_0^1 \frac{d}{dt} \mathbb{E}_{z_0\sim P_0, z_1\sim P_1}\left[ f(T(z_0,z_1,t)) \right] dt \right|.
\end{split}
\end{equation}
\end{proof}
Theorem \ref{fairmix} shows that the discrimination level $\Delta \textrm{DP}(f)$ is related to the derivative of the model $f$'s outcomes $f(\bar Z)$ w.r.t $t$. Specifically, we can reduce the resulting discrimination level by restricting the derivative. That is, if we have Eq. \ref{eq: fair constraint}, then we must have $\Delta \textrm{DP}(f) = 0$.

\section{Experiments}
In this section, we empirically verify the effectiveness of AFed. We compare the performance of our method with several baseline methods on four real-world datasets. After that, we demonstrate the performance gain of auxiliary classification head $h^a$ and the robustness of our method w.r.t different participant ratios.

\subsection{Experimental Setup}
We conducted experiments on four datasets, including three image datasets (UTKFace \cite{zhang2017age}, FairFace \cite{karkkainen2021fairface}, and CelebA \cite{liu2015faceattributes}) and one tabular dataset (Adult\footnote{http://archive.ics.uci.edu/ml/datasets/Adult}). In UTKFace and FairFace datasets, we set race as the sensitive attribute when training the model to predict the gender of a given image. For the CelebA dataset, the objective is to determine whether an image contains an object with wavy hair, and the data is split into male and female groups. For the Adult dataset, we aim to determine if a person's income exceeds 50K with gender as the sensitive attribute.


In our experiments, we preprocessed the dataset to increase the heterogeneity by 1) using the Dirichlet distribution to determine the quantity of data clients hold for the UTKface, Fairface, and Adult datasets; 2) creating multiple sets each containing images of twenty different celebrities from the CelebA dataset, each set is assigned to a different client. In addition, we also exaggerated the imbalance in the data.

\subsubsection{Configuration} We implemented the feature extractor $E$ using a five-layer convolutional neural network with a kernel size of 3, a stride of 2, and padding of 1. Each layer is followed by a batch normalization layer and a Rectified Linear Unit (ReLU) activation function. The two classification heads are realized by two fully connected layers with a hidden dimension of size 32. The generator and discriminator are implemented using Multi-Layer Perceptrons (MLPs) as \cite{zhu2021data}. The generator $G$ takes a noise vector $\epsilon$ of 32 dimensions and a one-hot sensitive attribute vector $a$ of 2 dimensions as input. The output dimension is the same as that of the feature extractor, which is 64 in our experiments. We use the Adam optimizer to update the feature extractor and two classification heads, as well as the RMSprop optimizer for the generator and discriminator. The learning rate is set to $0.005$ for the classification models and $0.0001$ for $G$, and $0.0003$ for $D$. The batch size equals the size of the local training dataset.

We simulate five clients, all clients are selected during every global communication round. The training epochs are set to 150, and the local training epochs are set to $T = 5$. 

\subsubsection{Baseline} To illustrate the effectiveness and superiority of AFed over debiasing methods that are solely based on local data, we select three baselines:
\begin{itemize}
    \item FedAvg: The canonical training method in FL without any considerations about fairness.
    \item FedReg: A naive way to learn fair model in FL. We add the fairness regularization term to the loss function to train a fair model on the client side and aggregate them on the server side.
    \item FairFed \cite{ezzeldin2021fairfed}: A reweighting-based mechanism, it amplifies the local debiasing via dynamically changing the aggregating weights of clients in each round. To make a fair comparison, we use Eq. \ref{eq: fair constraint} as a regularizer to local debiasing. 
    \item FairFB \cite{zeng2021improving}: Assigns different weights for each group, and the weight is optimized on the client side and aggregated by the server.

    \item FedGFT \cite{wang2023mitigating}: adapts the local objective function to the fairness regularization and uses strictly local summary statistics to minimize the fairness-penalized empirical loss of the global model.
    \item CML: We assemble all clients' data to form a centralized setting. We implement fair mix-up \cite{chuang2020fair} to train a fair model.
\end{itemize}

\subsection{Effectiveness of Proposed Method}
\begin{figure}[tp]
    \centering
        \begin{subfigure}[b]{0.24\textwidth}
        \includegraphics[width=\textwidth,height=\textwidth,keepaspectratio]{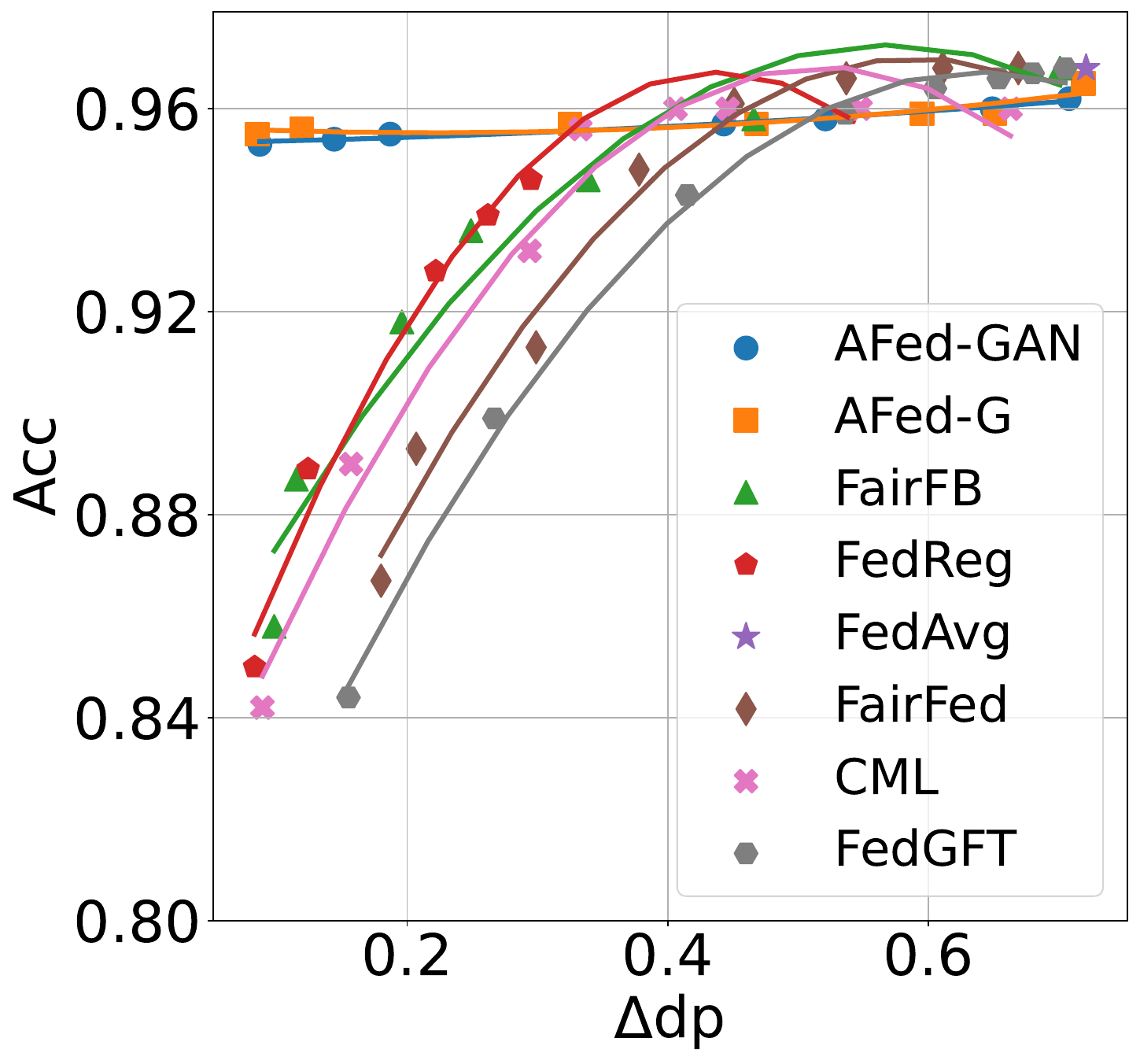}
        \caption{Adult}
    \end{subfigure}
        \hfill
    \begin{subfigure}[b]{0.24\textwidth}
    \includegraphics[width=\textwidth,height=\textwidth,keepaspectratio]{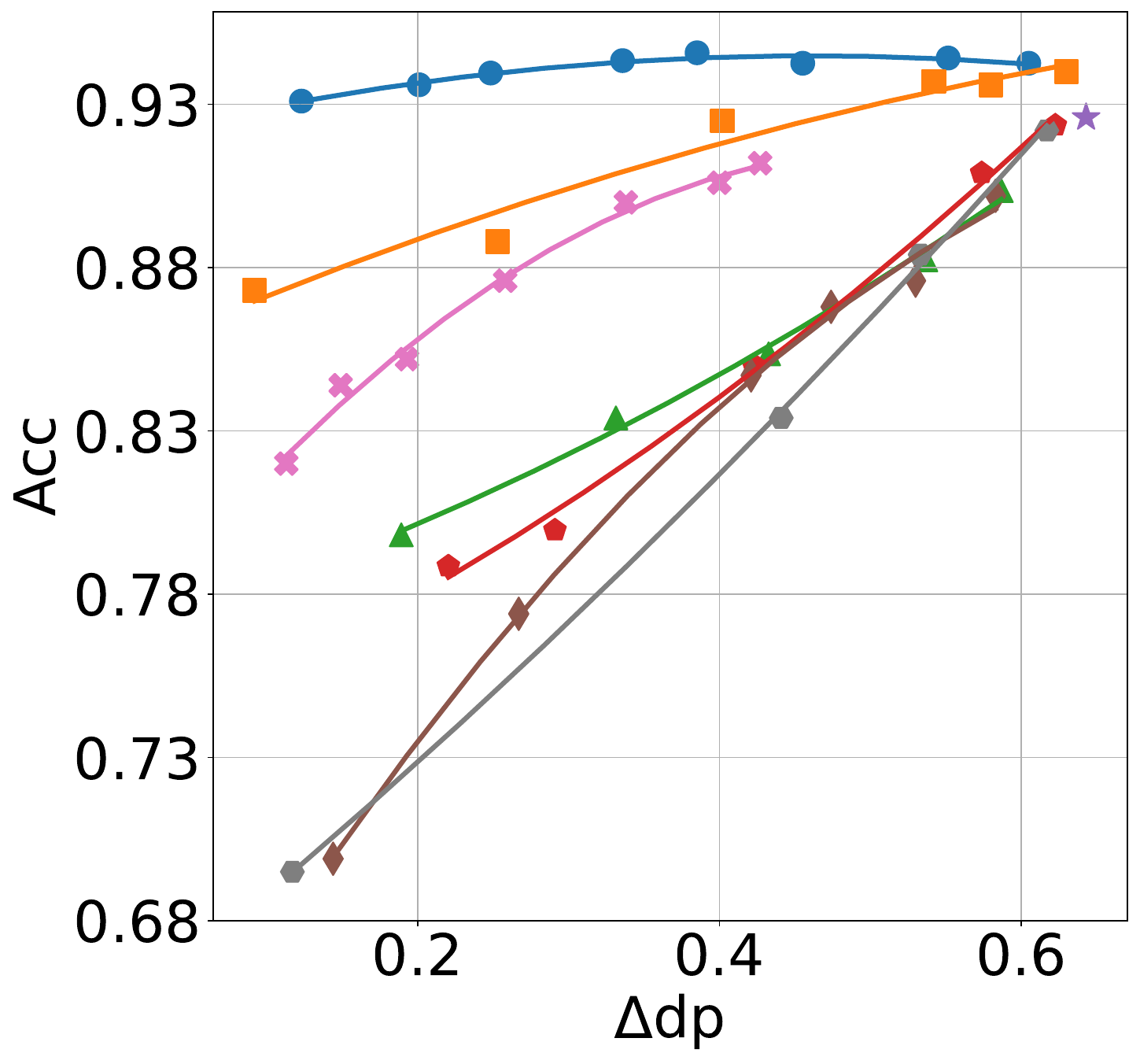}
        \caption{UTKface}
        \label{fig: utkface_acc_fair_res}
    \end{subfigure}
    \\
    \begin{subfigure}[b]{0.24\textwidth}
        \includegraphics[width=\textwidth,height=\textwidth,keepaspectratio]{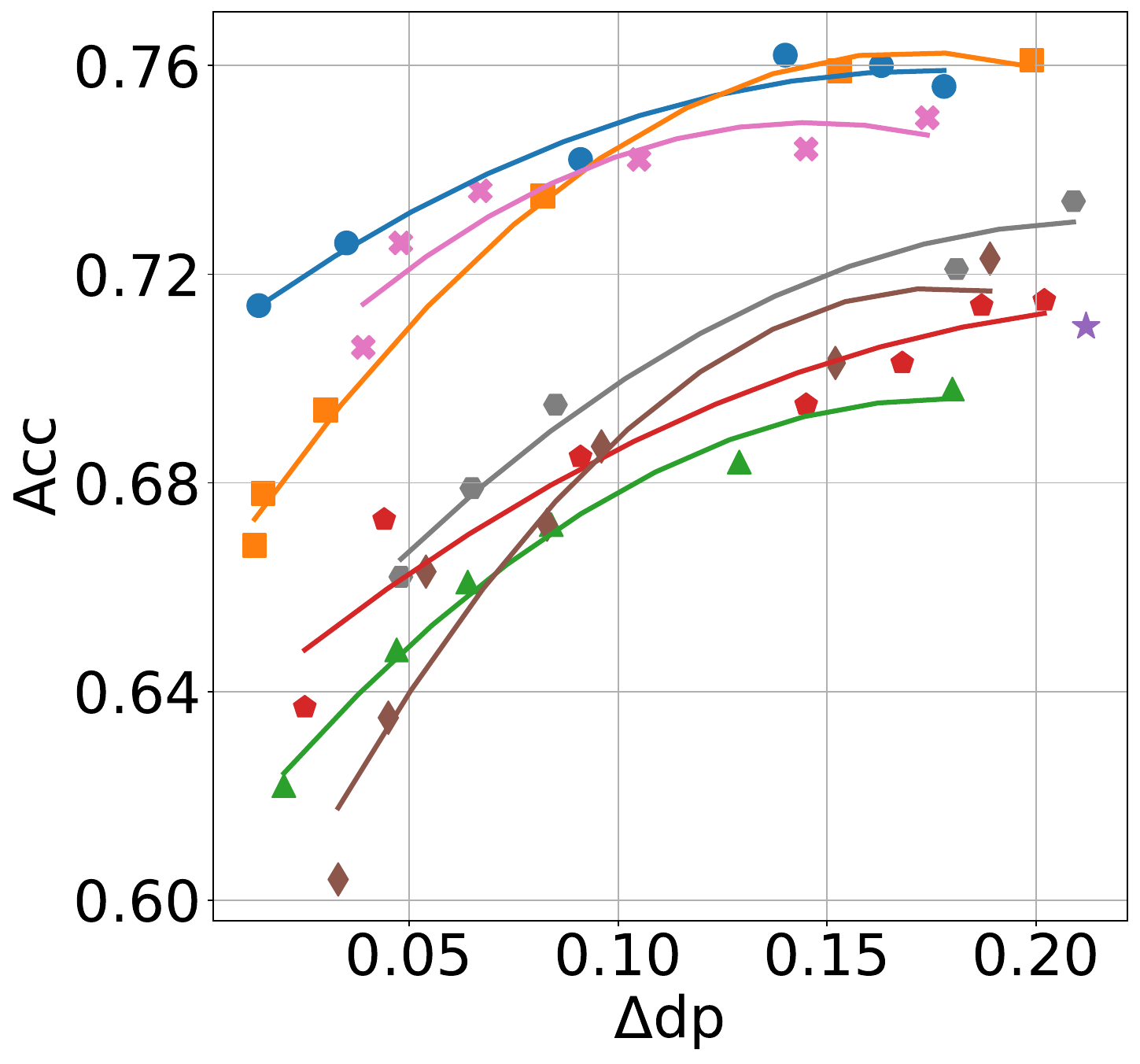}
        \caption{Fairface}
    \end{subfigure}    
    \begin{subfigure}[b]{0.24\textwidth}
        \includegraphics[width=\textwidth,height=\textwidth,keepaspectratio]{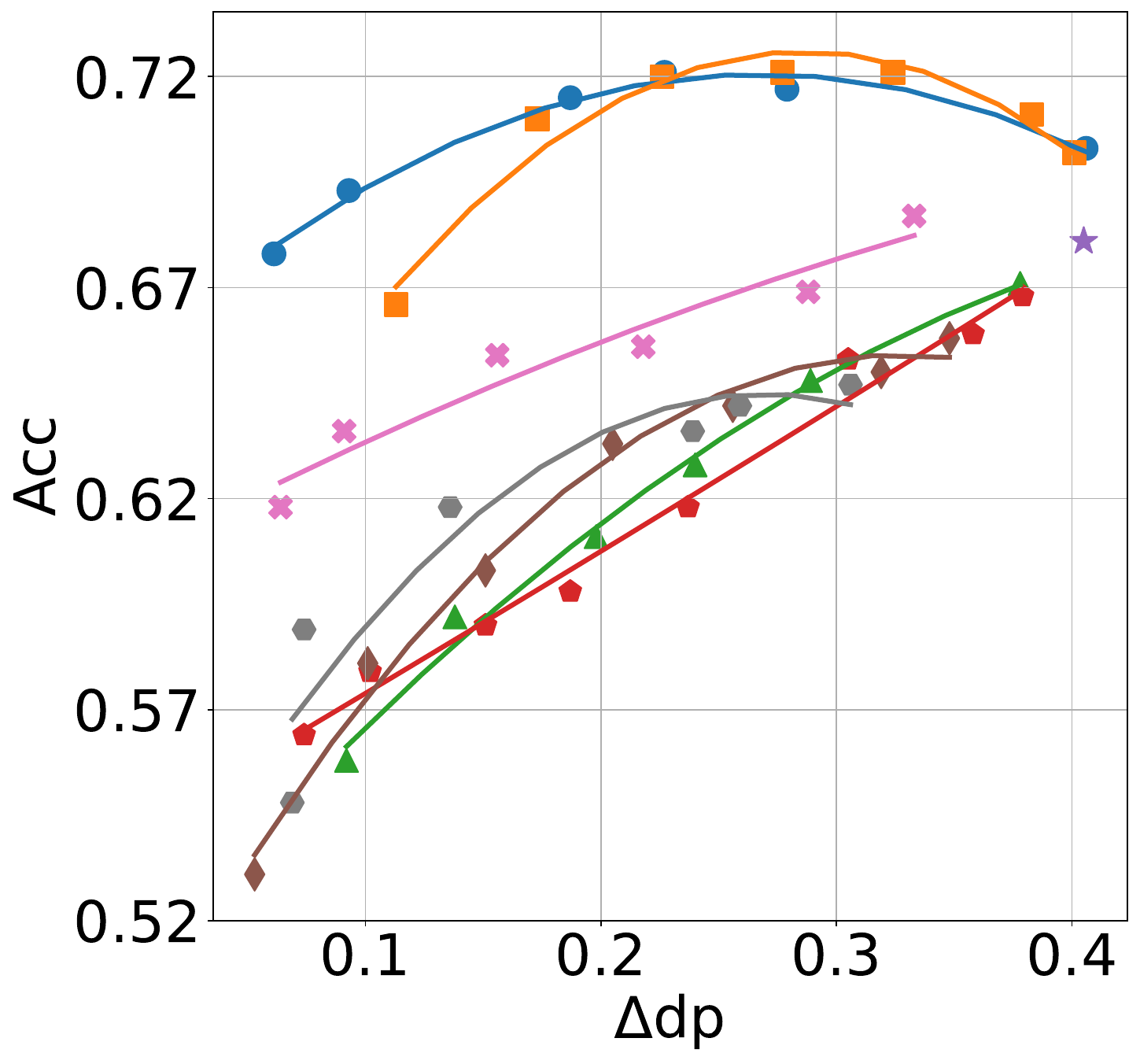}
        \caption{CelebA}
    \end{subfigure}
     \hfill
    \caption{The accuracy-fairness trade-off results on four real-world datasets. We alter the value of $\lambda$ in Eq. \ref{fig:Acc fair trade-off} to get the different accuracy and fairness metrics values. Results located in the upper left corner are preferable as they perform better in accuracy and fairness.}
    \label{fig:Acc fair trade-off}
\end{figure}

Fig. \ref{fig:Acc fair trade-off} shows the trade-offs between accuracy and $\Delta \textrm{dp}$ on three different datasets. All curves are obtained by altering the hyperparameter $\lambda$ in Eq. \ref{eq: base opt problem}, which controls the trade-off between accuracy and fairness. A smaller $\Delta \textrm{DP}$ means a fairer model. Only controlling $\Delta \textrm{DP}$ is trivial. However, our goal is to train a model that is both fair and accurate, which means we must control unfairness while still achieving acceptable accuracy performance. Therefore, we prioritize curves in the upper left corner over those in the lower right corner.

We present the averaged results across five runs in all experiments to mitigate the impact of randomness. Overall, several key observations can be drawn. 1) Both FairFed-Mix and CML significantly reduce discrimination but suffer from substantial accuracy degradation. As shown in Figure \ref{fig: utkface_acc_fair_res}, these methods reduce $\Delta DP$ by two-thirds, but at the cost of a more than 10\% accuracy drop for CML and over 20\% for FairFed-Mix. 2) AFed-GAN demonstrates superior performance compared to AFed-G and FairFed-Mix, particularly in the UTKface, Fairface, and CelebA datasets, where it outperforms FairFed-Mix and CML by a large margin. 3) On the Adult dataset, AFed-G achieves comparable results to AFed-GAN, while on more complex image-based tasks where the latent distribution is more diverse, AFed-GAN surpasses AFed-G. 4) Both AFed-G and AFed-GAN achieve higher accuracy than FedAvg, likely due to the generator's ability to extract and provide clients with valuable information about the global data distribution.


\begin{figure*}[htp]
    \centering
    \begin{subfigure}[b]{0.24\textwidth}
        \includegraphics[width=\textwidth]{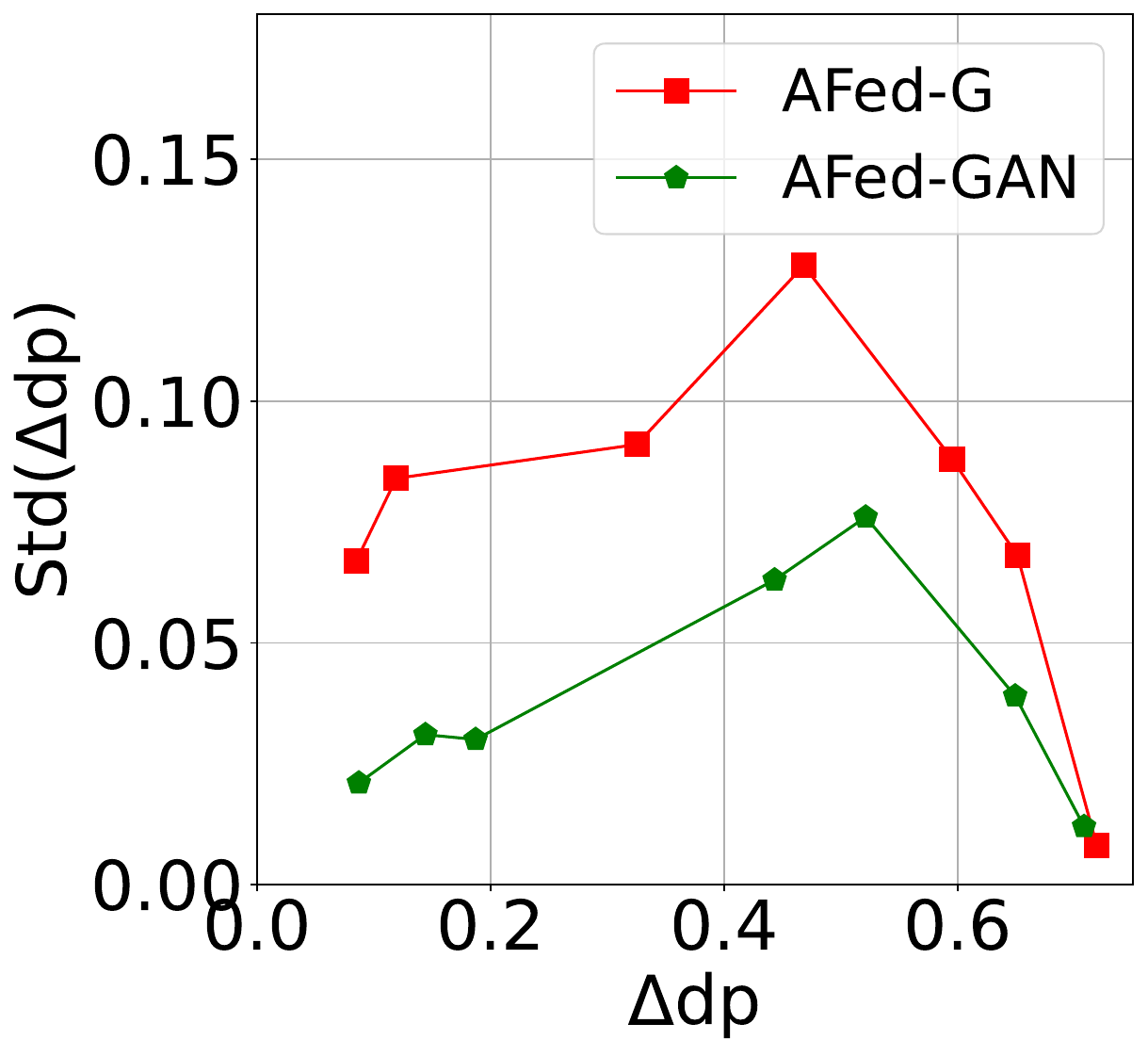}
        \caption{Adult} 
    \end{subfigure}
    \hfill
    \begin{subfigure}[b]{0.24\textwidth}
        \includegraphics[width=\textwidth]{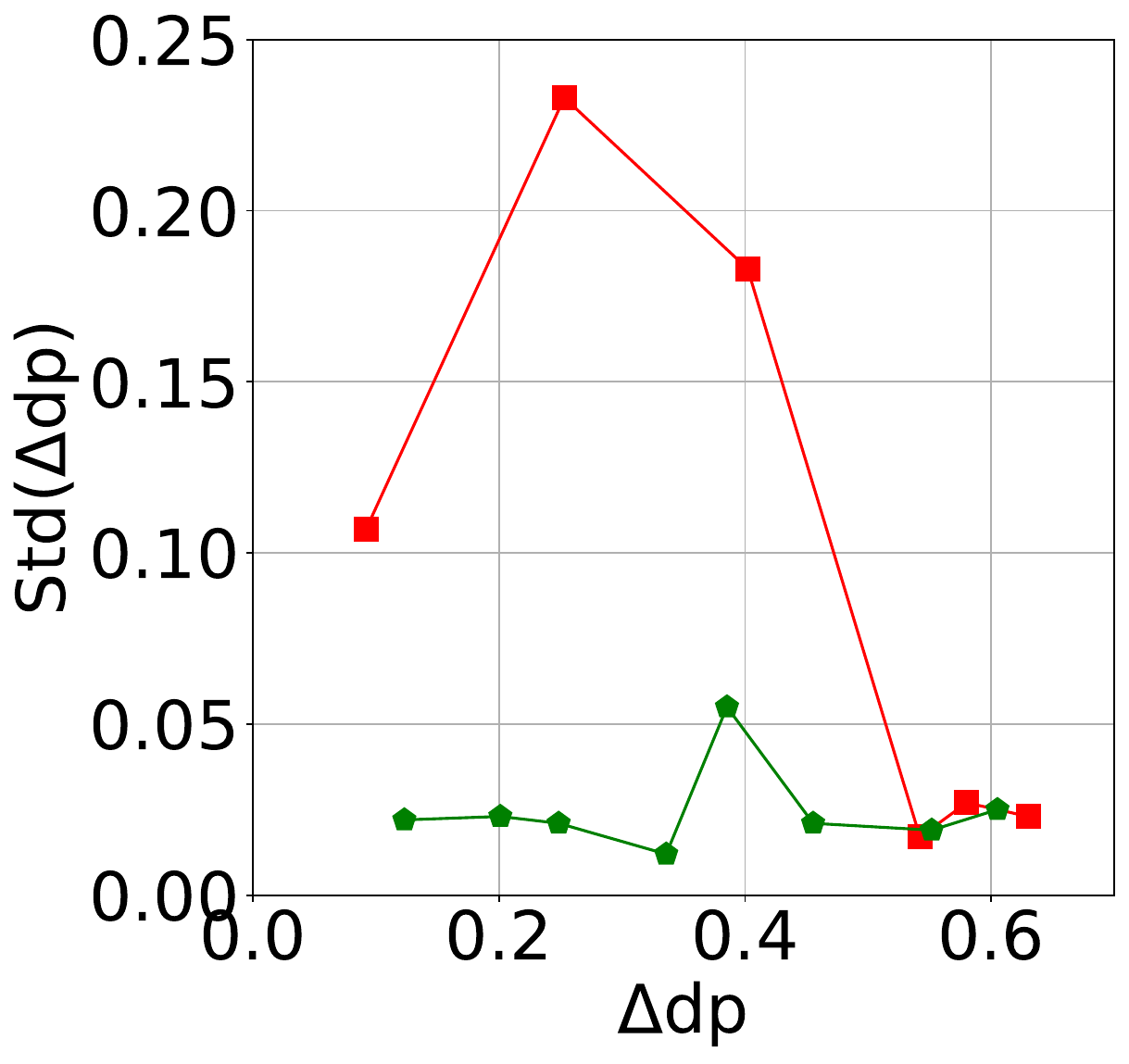}
        \caption{UTKface} 
    \end{subfigure}
    \hfill
    \begin{subfigure}[b]{0.24\textwidth}
        \includegraphics[width=\textwidth]{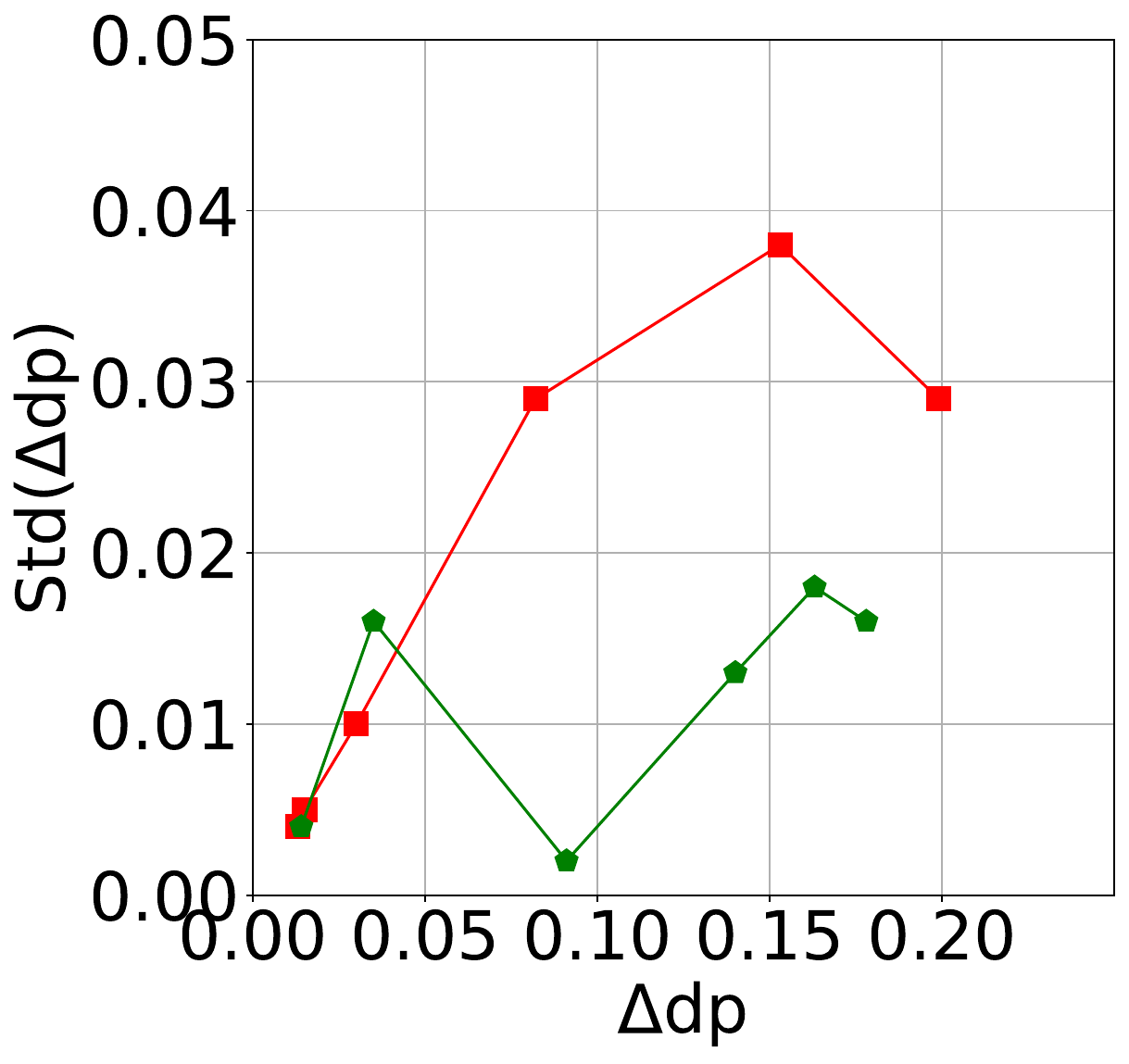}
        \caption{Fairface} 
    \end{subfigure} 
     \hfill
    \begin{subfigure}[b]{0.24\textwidth}
        \includegraphics[width=\textwidth]{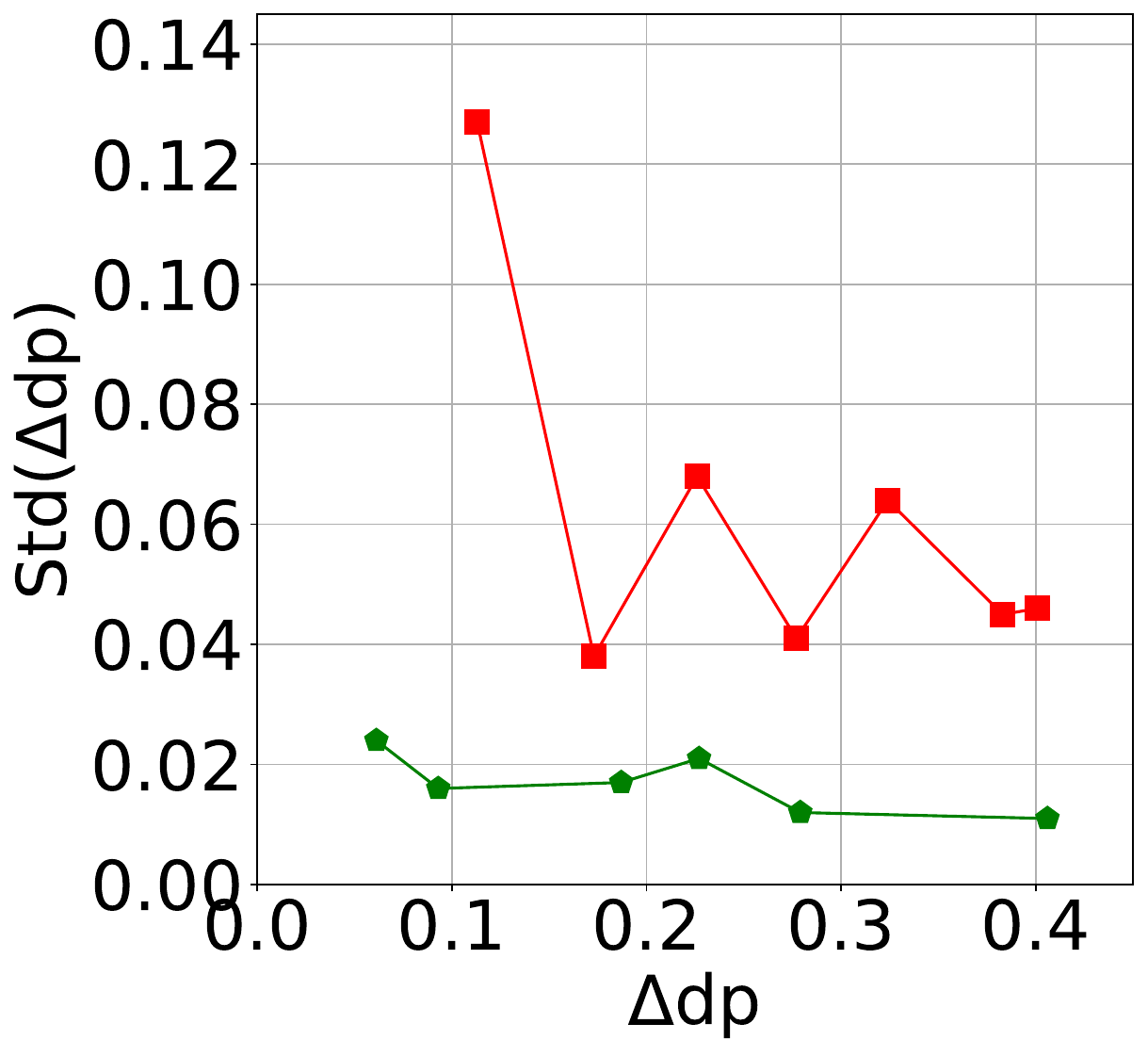}
        \caption{CelebA} 
    \end{subfigure}
    \caption{The change of standard variance of $\Delta \textrm{DP}$ w.r.t $\Delta \textrm{DP}$. Larger $\textrm{Std}(\Delta \textrm{DP})$ indicates higher fluctuation in debiasing result.}
    \label{fig:fair_std}
\end{figure*}

\begin{figure*}[htp]
    \centering
    \begin{subfigure}[b]{0.24\textwidth}
        \includegraphics[width=\textwidth]{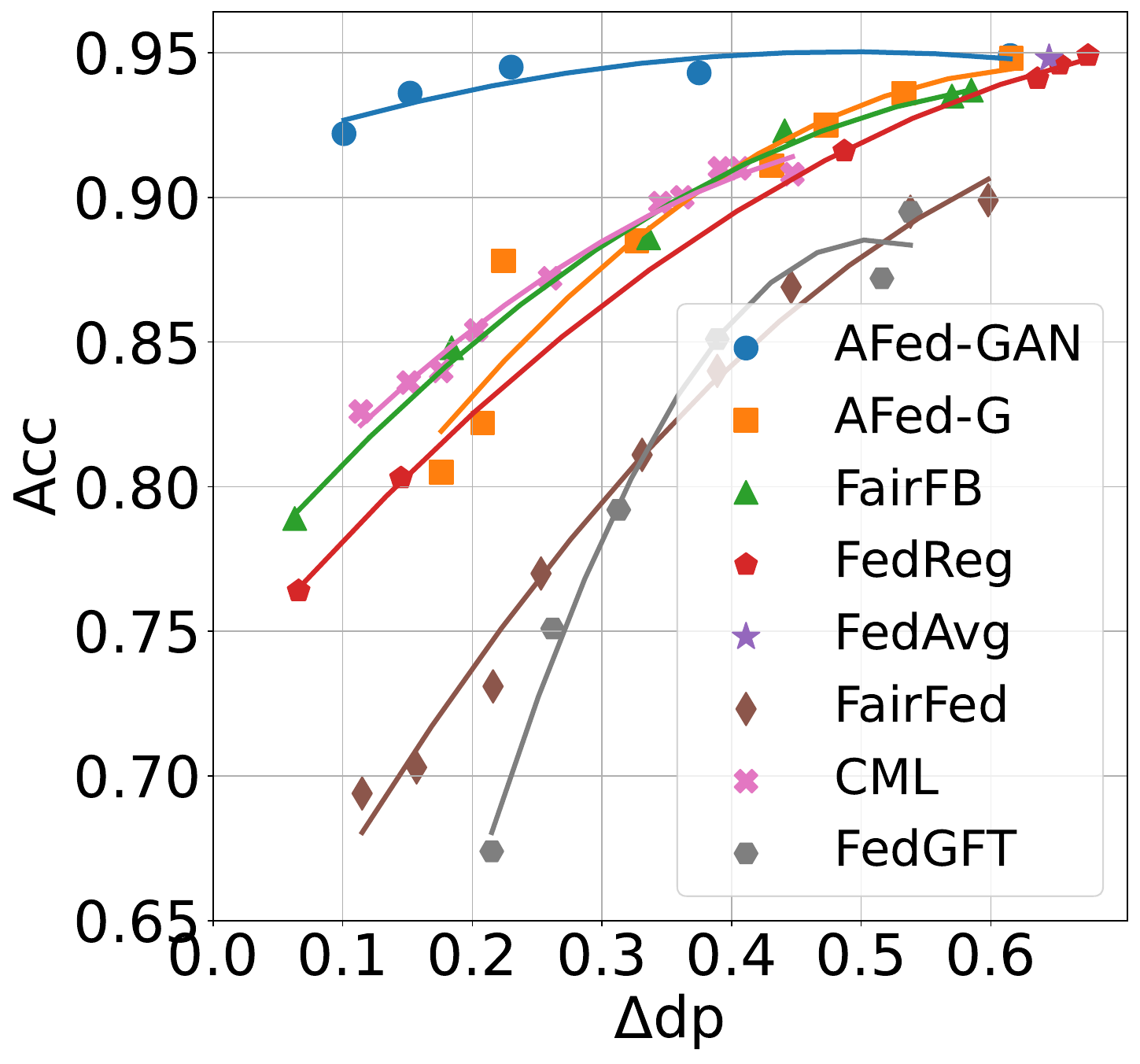}
        \caption{$r$=0.2} 
    \end{subfigure}
    \hfill
    \begin{subfigure}[b]{0.24\textwidth}
        \includegraphics[width=\textwidth]{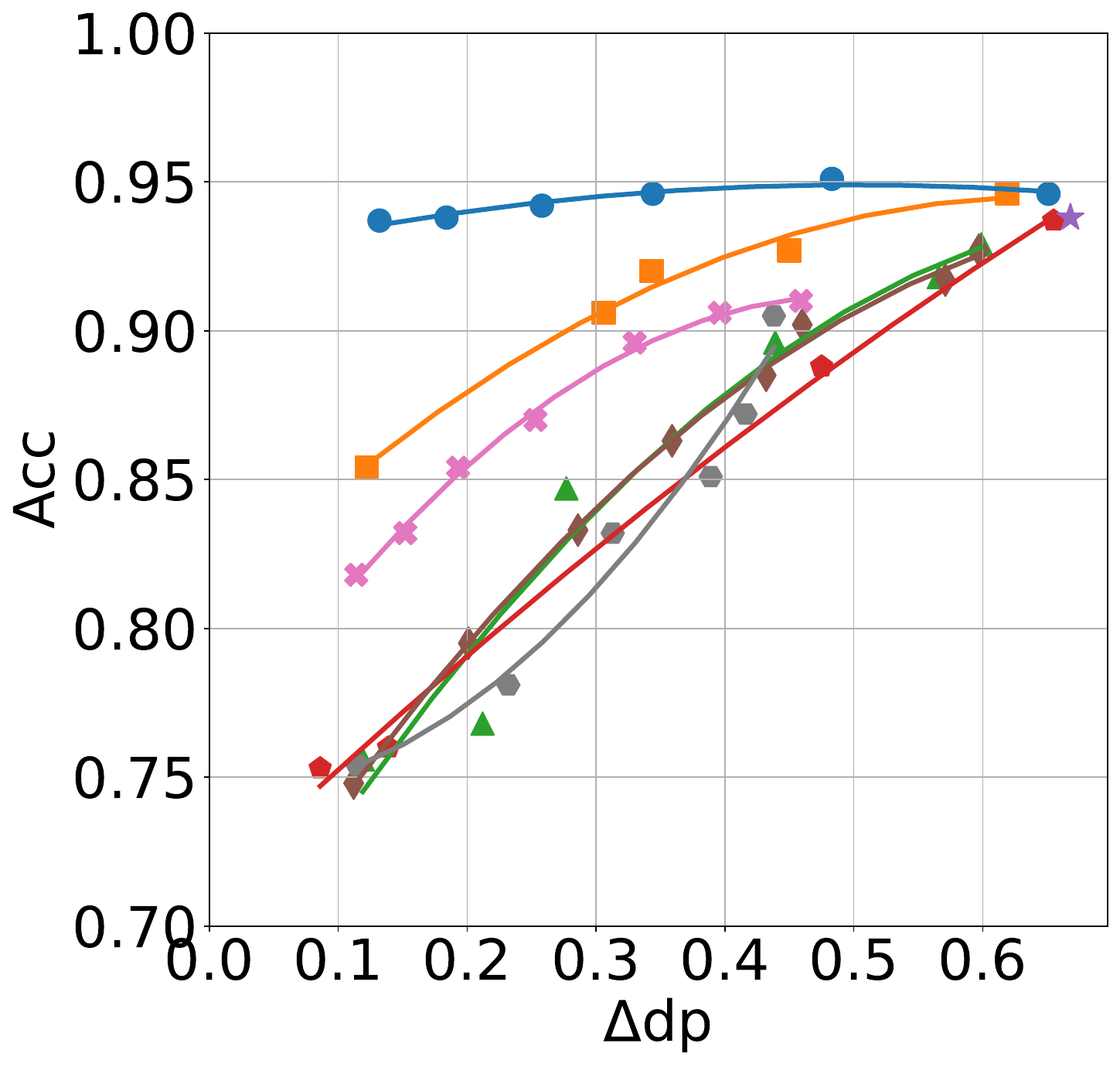}
        \caption{$r$=0.4} 
    \end{subfigure}    
    \begin{subfigure}[b]{0.24\textwidth}
        \includegraphics[width=\textwidth]{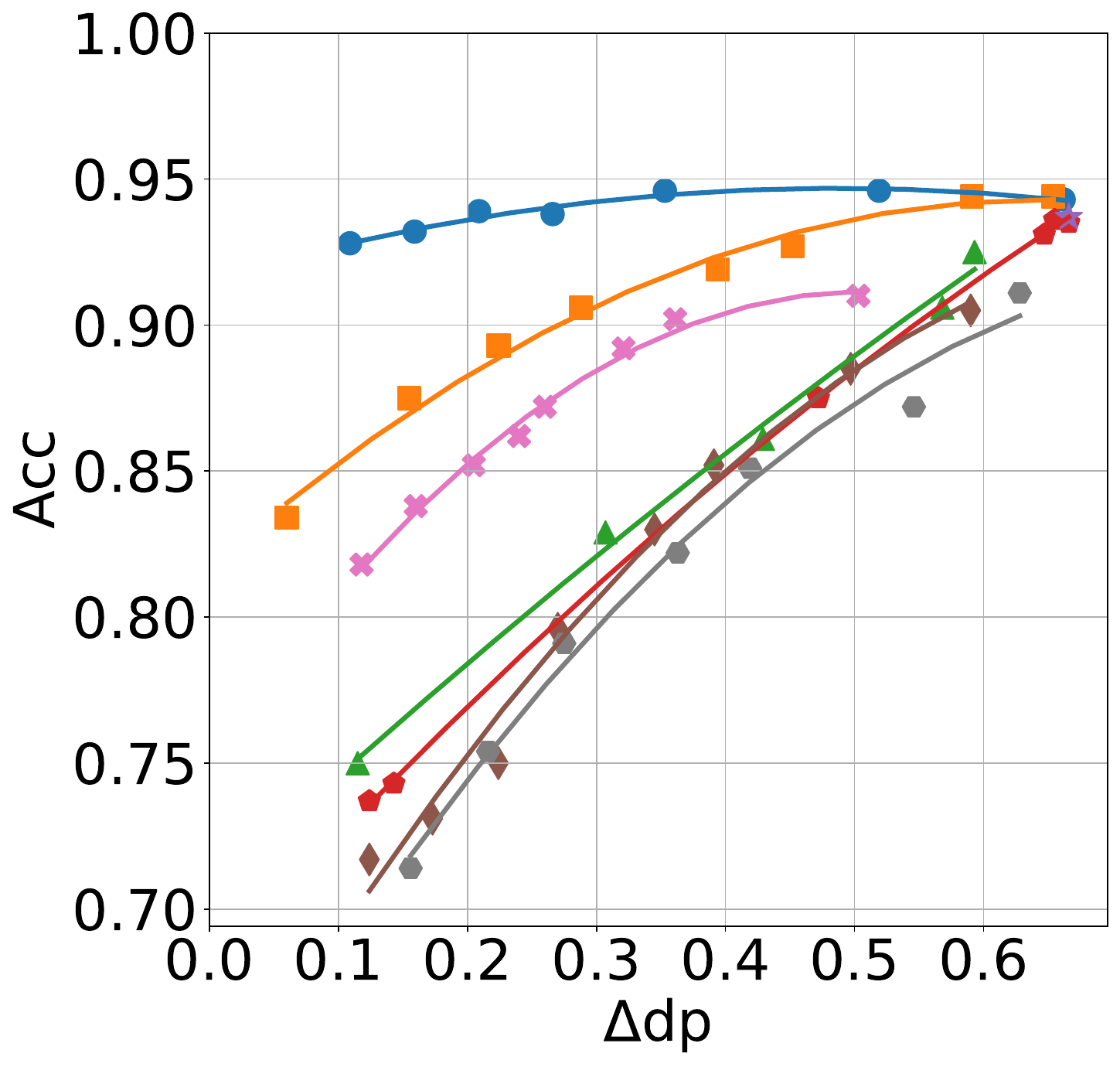}
        \caption{$r$=0.6} 
    \end{subfigure}
     \hfill
    \begin{subfigure}[b]{0.24\textwidth}
        \includegraphics[width=\textwidth]{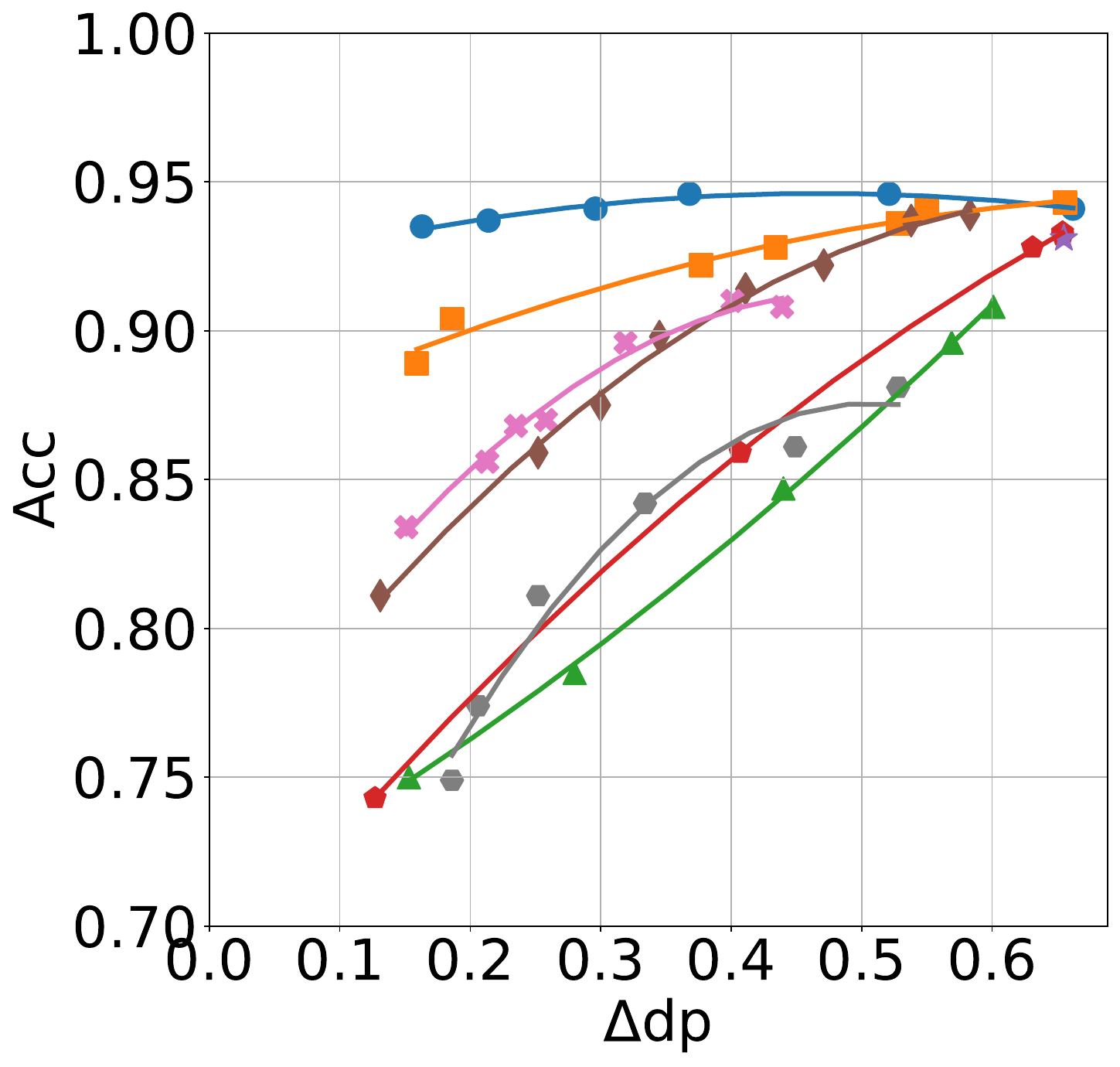}
        \caption{$r$=0.8} 
    \end{subfigure}
    \caption{Results on UTKface dataset. Four ratios $r$ are selected to show the impact of participant ratio on the proposed methods. Our methods consistently outperform baseline methods.}
    \label{fig:different r}
\end{figure*}
\subsection{Comparison between Generator and GAN Approach}
Our analysis in Section V indicates the conditional generator in AFed-G surfers lack of diversity problem, which would result in inferior performance to the GAN approach. Our experiment results confirm this finding. As we can see from Fig. \ref{fig:Acc fair trade-off}, for the simple task, i.e., task on the Adult dataset, AFed-GAN and AFed-G demonstrate similar performance. While on complicated tasks, i.e., tasks on UTKface, Fairface and CelebA, AFed-GAN outperforms AFed-G. As $\Delta \textrm{DP}$ decrease, the accuracy gap between AFed-GAN and AFed-G increase. 

We contribute this to the sample diversity. One drawback of AFed-G is optimizing Eq. \ref{eq: AFed-G-opt} will decrease the diversity of the generated samples. Debiasing with those samples will work at the earlier stage, and as the $\Delta \textrm{DP}$ decreases, the model can make fair predictions on most samples except for some \quotes{hard} samples. And, due to the lack of diversity, the AFed-G algorithm fails to generate samples similar to those \quotes{hard} samples that contribute to unfairness. As a comparison, AFed-GAN can generate diverse samples, thus, better debiasing performance. 

Our claim is supported by Fig.\ref{fig:fair_std}. It shows the change of $\textrm{std}(\Delta\textrm{DP})$ w.r.t different $\Delta\textrm{DP}$. For all four tasks, AFed-G shows an overall higher standard variation than AFed-GAN. This confirms our argument that due to the lack of effective feedback, data generated by AFed-G lacks diversity. Hence, leading to a higher fluctuation in the debiasing result.

\subsection{Effects of Auxiliary Classification Head}
We remove the auxiliary classification head $h^a$ and rerun the experiment on the UTKface dataset. The results are shown in Fig. \ref{fig: w/o ha}. The performance of both AFed-GAN and AFed-G drops after removing $h^a$. This confirms our previous analysis that $h^a$ helps the feature extractor extract more informative representations and keep information about label $y$ and attribute $a$. Which consequently benefits the generator training. Nevertheless, our methods outperform baseline methods without an auxiliary classification head.
\begin{figure}[htp]
    \centering
    \includegraphics[width=0.24\textwidth]{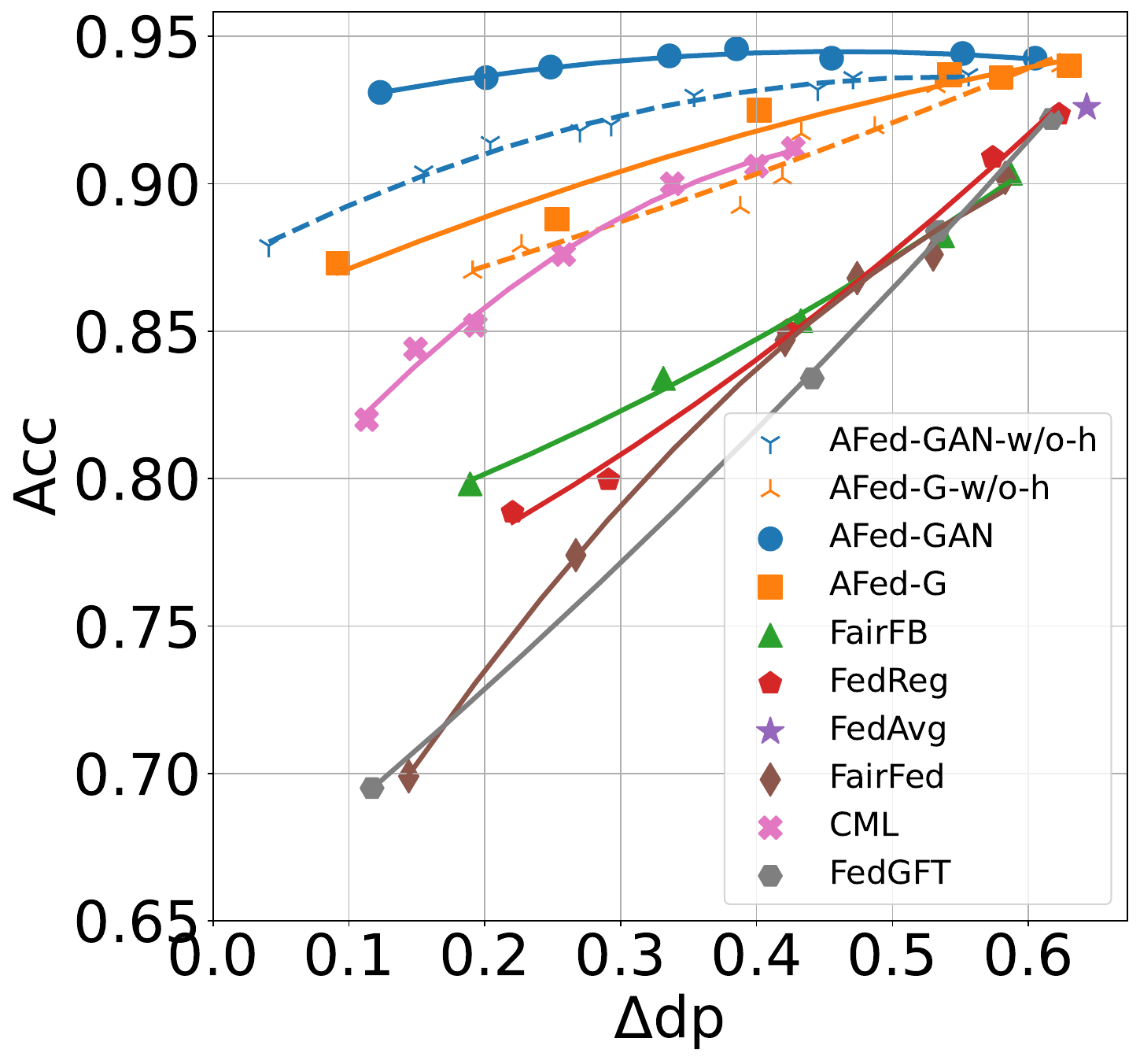}
    \caption{AFed-G and AFed-GAN with and w/o $h^a$. Results are obtained on the UTKface dataset.}
    \label{fig: w/o ha}
\end{figure}

\subsection{Impacts of Participant ratio}
In each round of FL, only part of the clients is selected to participate in the current round. To investigate the impact of participant ratio $r$ on proposed methods, we performed an additional experiment on the UTKface dataset. We vary participant ratio $r$ from 0.2 to 0.8. As we can see from Fig. \ref{fig:different r}, AFed-GAN consistently outperforms AFed-G and the two baseline methods. Notice that when $r=0.2$, AFed-G didn't show much improvement. In addition, the improvement of AFed-G increased with the increase of $r$. 

Training a conditional generator on the server side needs the information related to the decision boundary embedded clients' classification head $h_k^a$. The decision boundary information may be inaccurate when there are only limited clients. As a result, the server fails to train a good generator $G$ that can produce fake latent features of high quality (i.e., lie within the decision boundary). On the contrary, the generator in AFed-GAN is free of such problems. As we can see, the performance of AFed-GAN barely changes w.r.t $r$.

\section{Conclusion}
In this paper, we propose AFed-G and AFed-GAN, two algorithms designed to learn a fair model in FL without centralized access to clients' data. We first employ a generative model to learn the clients' data distribution, which is then shared with all clients. Subsequently, each client's data is augmented with synthetic samples, providing a broader view of the global data distribution. This approach helps derive a more generalized and fair model. Both empirical and theoretical analyses validate the effectiveness of the proposed method. One potential future application of our methods is the case of multiple sensitive attributes, where the data is divided into several groups. Additionally, the proposed framework could be extended to FL scenarios where each client uses a different model architecture, provided they share the same latent feature dimension.
 
\ifCLASSOPTIONcaptionsoff
  \newpage
\fi
\printbibliography

\begin{IEEEbiography}
[{\includegraphics[width=1in,height=1.25in,clip,keepaspectratio]{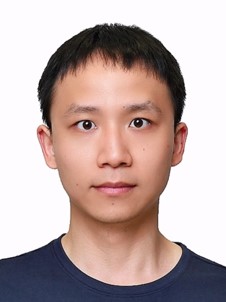}}] 
{Huiqiang Chen} received the B.E. degree from Hunan University, China, in 2016, M.S. degree form Zhejiang University, China, in 2019, and the Ph.D. degree from the School of Computer Science, University of Technology Sydney, in 2025. His research interests center on trustworthy machine learning and cybersecurity, with a focus on AI security, privacy-preserving techniques, and fairness-aware algorithms in deep neural networks.
\end{IEEEbiography}

\begin{IEEEbiography}
[{\includegraphics[width=1in,height=1.25in,clip,keepaspectratio]{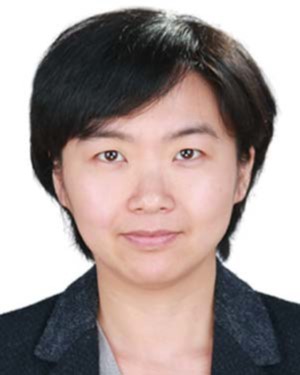}}] 
{Tianqing Zhu} received the B.Eng. and M.Eng. degrees from Wuhan University, Wuhan, China, in 2000 and 2004, respectively, and the Ph.D. degree from Deakin University, Sydney, Australia, in 2014. She was a Lecturer with the School of Information Technology, Deakin University, from 2014 to 2018, and an Associate Professor with the University of Technology Sydney, Ultimo, NSW, Australia. She is currently a Professor with the Faculty of Data Science, City University of Macau, Macau, SAR, China. Her research interests include cyber security and privacy in artificial intelligence (AI).
\end{IEEEbiography}

\begin{IEEEbiography}
[{\includegraphics[width=1in,height=1.25in,clip,keepaspectratio]{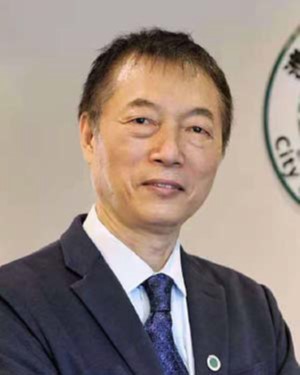}}] 
{Wanlei Zhou} received the PhD degree in computer science and Engineering from The Australian National University, Canberra, Australia, in 1991 and the DSc degree (a higher Doctorate degree) from Deakin University in 2002. He is currently the vice-rector (Academic Affairs) and dean of the Institute of Data Science, City University of Macau, Macao, China. He has authored or co-authored more than 400 papers in refereed international journals and refereed international conference proceedings, including many articles in IEEE Transactions and journals. His main research interests include security, privacy, and distributed computing.
\end{IEEEbiography}

\begin{IEEEbiography}
[{\includegraphics[width=1in,height=1.25in,clip,keepaspectratio]{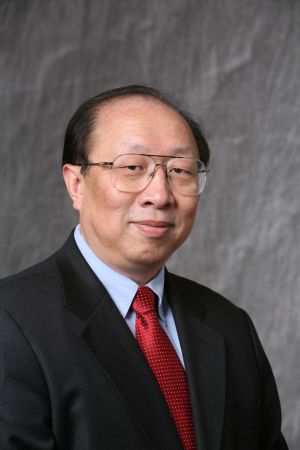}}] 
{Wei Zhao} An internationally renowned scholar, Professor Wei Zhao has served important leadership roles in academia, including, currently the Provost and Chair Professor at the Shenzhen University of Advanced Technology, and previously the eighth Rector (i.e., President) of the University of Macau, the Dean of Science at Rensselaer Polytechnic Institute, the Director for the Division of Computer and Network Systems in the U.S. NSF, the Chief Research Officer (i.e., VPR) at the American University of Sharjah, the Chair of Academic Council at the Shenzhen Institute of Advanced Technology, and the Senior Associate Vice President for Research at Texas A\&M University. An IEEE Fellow, Professor Zhao has made significant contributions in cyber-physical systems, distributed computing, real-time systems, computer networks, and cyberspace security. He led the effort to define the research agenda of and to create the very first funding program for cyber-physical systems when he served as the NSF CNS Division Director in 2006. His research group has received numerous awards. Their research results have been adopted in the standard of DOD SAFENET. Professor Zhao was honored with the Lifelong Achievement Award by the Chinese Association of Science and Technology, and the Overseas Achievement Award by the Chinese Computer Federation. He has been conferred honorable doctorates by 12 universities in the world and academician of the International Eurasian Academy of Sciences.
\end{IEEEbiography}
\end{document}